%% file: arxiv/main.tex
\documentclass[11pt]{article}
\usepackage[left=1in,top=1in,right=1in,bottom=1in,letterpaper]{geometry}
\usepackage[colorlinks,linkcolor=red,filecolor=blue,citecolor=blue,urlcolor=blue]{hyperref}
\usepackage{mathtools}
\usepackage{url}
\usepackage{amsthm,amsmath,amssymb,amscd}
\usepackage{mathrsfs}
\usepackage{amsfonts}
\usepackage[utf8]{inputenc}
\usepackage{graphicx}
\usepackage{float}
\usepackage{bbding}
\everymath{\displaystyle}
\usepackage{indentfirst}
\usepackage{subcaption}
\usepackage{enumerate}
\usepackage{bm,bbm}
\usepackage[shortlabels]{enumitem}
\usepackage{algorithm,algorithmic}
\usepackage{dsfont}
\usepackage{booktabs}
\usepackage{array,siunitx}
\usepackage{xspace}
\usepackage{pifont}
\usepackage{xcolor}
\usepackage{colortbl}
\usepackage{cleveref}
\usepackage{mathpazo}
\usepackage{natbib}
\numberwithin{equation}{section}

\let\amsEqref\eqref
\input{math_commands}

\input{def}
\let\eqref\amsEqref

\renewcommand{\KL}{\operatorname{KL}}

\allowdisplaybreaks
\usepackage{afterpage}
\definecolor{fitink}{HTML}{1F5F99}
\definecolor{fitmetric}{HTML}{DDF1EA}
\definecolor{fitmetricink}{HTML}{176B57}

\title{How RL Reshapes LLM Reasoning: Transferability, Coverage, and Scaling Laws}

\author{Ziheng Cheng\thanks{Equal contribution.}, Yixiao Huang\footnotemark[1], Hanlin Zhu, Somayeh Sojoudi\\[0.5ex]
\normalsize University of California, Berkeley\\
\normalsize \texttt{\char123 ziheng\_cheng,yixiaoh,hanlinzhu,sojoudi\char125 @berkeley.edu}}
\date{}

\begin{document}

\maketitle

\begin{abstract}
Recent studies on reinforcement learning (RL) report seemingly conflicting evidence about large language model (LLM) reasoning. Training on mathematics can improve performance in other domains, yet gains in Pass@1 can coincide with lower Pass@$N$ than the base model. This raises a fundamental question: does RL expand an LLM's reasoning boundary, or merely reweight its existing reasoning space? We revisit these phenomena across Qwen and Gemma model families, showing both cross-domain gains and forgetting, while coverage at large sampling budgets increases on some tasks and decreases on others.
Detailed analysis of solution traces before and after RL indicates a shift in the reasoning strategies the model employs, motivating a two-stage autoregressive policy model that separates \emph{strategy selection} from problem-specific execution.
Within this framework, we prove how RL's implicit bias reshapes strategy preferences, allowing gains on some tasks while suppressing strategies required by others. This mechanism can also broaden or narrow coverage at a given sampling budget even without expanding strategy support.
We further provide theoretical justifications for log-sigmoid and log-linear scaling laws in RL compute, and evaluate their predictive power. 
Together, these results connect changes in strategy selection to cross-domain transfer, reasoning coverage, and compute scaling.
\end{abstract}

% Keep the introduction on page 1 and insert the teaser at the start of page 2.
\afterpage{\clearpage\input{arxiv/figures/how_rl_reshapes_teaser_figure}}

\input{arxiv/sections/introduction.tex}

\input{arxiv/sections/empirical_puzzles.tex}

\input{arxiv/sections/two_stage_model.tex}

\input{arxiv/sections/transfer_forgetting.tex}

\input{arxiv/sections/passn_boundary.tex}

\input{arxiv/sections/scaling_laws.tex}

\input{arxiv/sections/related_work.tex}

\input{arxiv/sections/conclusion.tex}
\input{arxiv/sections/ack.tex}

\bibliographystyle{plainnat}
\bibliography{ref,arxiv/related_work_refs}

\newpage
\appendix
\crefalias{section}{appendix}
\crefalias{subsection}{appendix}
\input{arxiv/appendix/transfer_forgetting_proofs.tex}

\input{arxiv/appendix/passn_boundary_proof.tex}

\input{arxiv/appendix/scaling_laws_proofs.tex}

\input{arxiv/appendix/experiment_details.tex}

\input{arxiv/appendix/reasoning_route_analysis.tex}

\input{arxiv/appendix/empirical_case_studies.tex}

\end{document}

%% file: math_commands.tex
\usepackage{amsmath,amsfonts,bm}

\def\eqref#1{equation~\ref{#1}}
\def\1{\bm{1}}

\DeclareMathAlphabet{\mathsfit}{\encodingdefault}{\sfdefault}{m}{sl}
\SetMathAlphabet{\mathsfit}{bold}{\encodingdefault}{\sfdefault}{bx}{n}

\newcommand{\E}{\mathbb{E}}

\newcommand{\R}{\mathbb{R}}

\newcommand{\KL}{D_{\mathrm{KL}}}

\DeclareMathOperator*{\argmin}{arg\,min}

%% file: def.tex
\newtheorem{thm}{Theorem}[section]

\newtheorem{Assumption}{Assumption}
\newtheorem{lemma}[thm]{Lemma}
\newtheorem{coro}[thm]{Corollary}

\theoremstyle{definition}

\crefname{equation}{eq.}{eqs.}
\Crefname{equation}{Eq.}{Eqs.}
\crefname{Assumption}{Assumption}{Assumptions}
\Crefname{Assumption}{Assumption}{Assumptions}
\crefname{thm}{Theorem}{Theorems}
\Crefname{thm}{Theorem}{Theorems}
\crefname{prop}{Proposition}{Propositions}
\Crefname{prop}{Proposition}{Propositions}
\crefname{definition}{def.}{defs.}
\Crefname{definition}{Def.}{Defs.}
\crefname{algorithm}{alg.}{algs.}
\Crefname{algorithm}{Alg.}{Algs.}
\crefname{figure}{Figure}{Figures}
\Crefname{figure}{Figure}{Figures}
\crefname{section}{Section}{Sections}
\Crefname{section}{Section}{Sections}

%% file: arxiv/figures/how_rl_reshapes_teaser_figure.tex
% Inserted after page 1; [H] keeps the teaser at that exact position.
\begin{figure}[H]
\centering
\includegraphics[width=\linewidth]{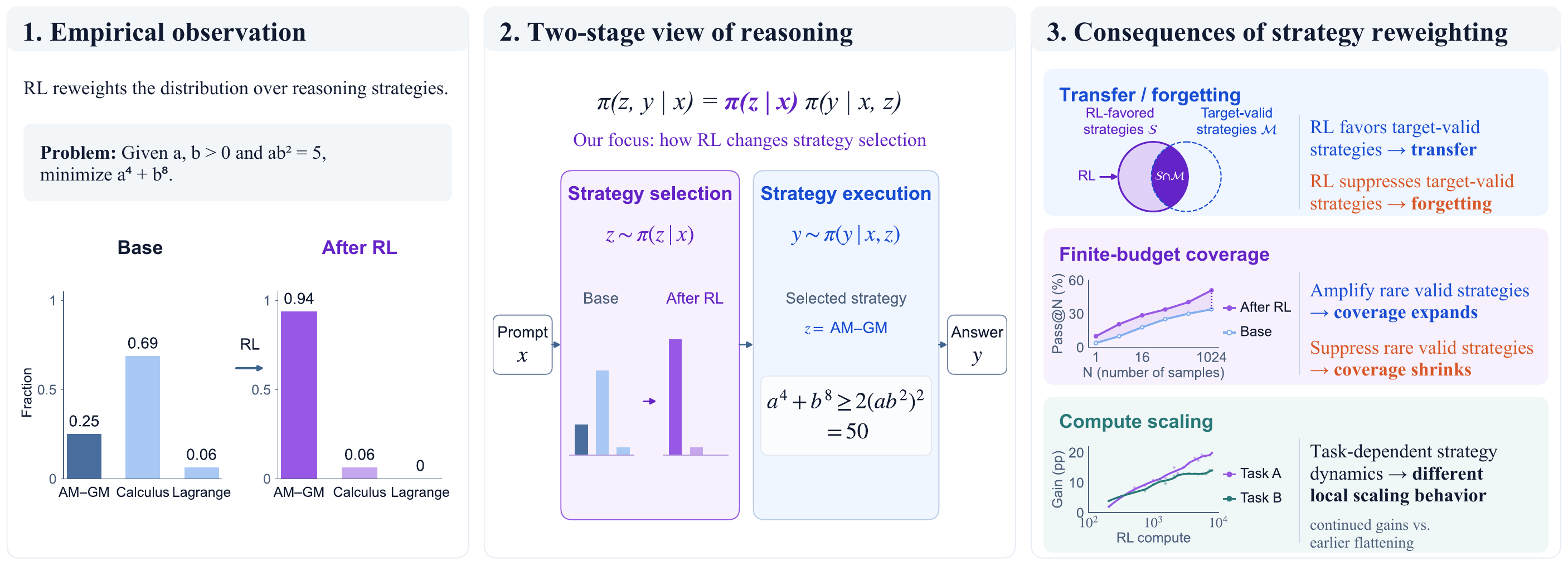}
\caption{\textbf{RL reweights the distribution over reasoning strategies.} Left: for the same problem, RL changes the relative frequencies of existing reasoning strategies. Center: a two-stage factorization separates strategy selection from execution, allowing us to isolate how RL reshapes strategy preferences. Right: the theory connects strategy reweighting to transfer and forgetting, finite-budget coverage, and compute scaling.}
\label{fig:teaser}
\end{figure}

%% file: arxiv/sections/introduction.tex
\section{Introduction}
\label{sec:introduction}

Reinforcement learning (RL) has improved the performance of large language models (LLMs) on reasoning tasks~\citep{shao2024deepseekmath,deepseekai2025r1,yu2025dapo}. RL updates a model using reward feedback on sampled responses, rather than training it to imitate demonstrations. Its success on mathematics and coding raises a broader question about how RL reshapes LLM reasoning.

What accounts for these gains remains debated. RL may teach a model new ways to reason, or make it more likely to use reasoning it can already produce. For example, RL training on mathematical problems can improve performance in other domains~\citep{huan2025transferability}, but such gains are compatible with either interpretation. Sampling provides another perspective. Gains in one-shot accuracy (Pass@1) can accompany lower Pass@$N$, the fraction of problems solved within $N$ sampled attempts~\citep{yue2504does}, while prolonged training has been reported to broaden coverage~\citep{liu2025prorl}. These findings motivate the central question. Does RL expand an LLM's reasoning boundary, or merely reweight its existing reasoning space?

Our experiments show that training can improve performance in its own domain while reducing performance on other tasks or lowering Pass@1024 (\Cref{fig:transfer-forgetting,fig:passn-qwen-gemma}).

Comparing Qwen trajectories before and after RL indicates a shift in the model's preferences among reasoning strategies (\Cref{sec:reasoning-mode-analysis}). This motivates a two-stage policy model that separates \emph{strategy selection} from problem-specific execution. Within this framework, we characterize how the implicit bias of RL reshapes strategy preferences, improving performance on some tasks while suppressing strategies required by others. We also contrast this mechanism with supervised fine-tuning (SFT). Our focus on how strategy preferences generalize across tasks complements the analyses of training-prompt sharpening and coverage~\citep{chen2026reshaping,chen2025coverage}.
Crucially, this reweighting mechanism explains coverage shifts without introducing new strategies. Pass@$N$ at large $N$ is especially sensitive to problems with low success probabilities. Amplifying or suppressing rare valid strategies can broaden or narrow benchmark coverage, even as Pass@1 improves.

Recent studies also examine how performance scales with RL compute~\citep{khatri2025scalerl,meta2026musespark}. Our framework connects strategy selection to scaling laws, showing that selection dynamics admit log-sigmoid and log-linear approximations.

Our contributions are summarized as follows:
% \paragraph{Contributions.}
\begin{itemize}
    \item We document task-specific RL transfer and forgetting alongside divergent coverage shifts, with outcomes dictated by the training data, evaluation benchmark, and base model (\Cref{sec:empirical}).
    \item We analyze Qwen solution trajectories and find changes in which strategies are sampled for a question, motivating a separation of strategy selection from execution (\Cref{sec:reasoning-mode-analysis,sec:two-stage}).
    \item Theoretically, we characterize RL's implicit bias in strategy selection and derive conditions for transfer, forgetting, and expansion or shrinkage of coverage at a given sampling budget (\Cref{sec:transfer,sec:passn}).
    \item We derive local approximation bounds for log-sigmoid and log-linear scaling laws and evaluate how well the fitted curves predict performance later in training, showing that predictive reliability varies with the training mixture and evaluation domain (\Cref{sec:scaling}).
\end{itemize}
% RL gains should therefore be judged by which strategies remain useful and accessible on target tasks, not by training-domain accuracy or Pass@1 alone.

%% file: arxiv/sections/empirical_puzzles.tex
\section{Empirical Evidence on Transfer and Reasoning Coverage}
\label{sec:empirical}

The extent to which RL transfers beyond its training domain and expands the reasoning boundary remains actively debated. In this section, we revisit these phenomena across various model families and task domains to show that these seemingly contradictory findings can be reconciled. 
% within different training and evaluation regimes.

\subsection{Experimental setup}
\label{sec:experimental-setup}

\paragraph{Models and datasets.}
Our experiments use Qwen3-14B-Base~\citep{yang2025qwen3} and Gemma-3-12B-IT~\citep{gemmateam2025gemma3} as base models. We train both models on 48K mathematics problems from UniReason~\citep{huan2025transferability} (Math RL).
For Qwen, we also train a variant on this dataset augmented with 5K curated coding problems from Eurus-2-RL~\citep{cui2025process} (Math-Code RL).
For comparison, we also train both models with SFT using 40K mathematics demonstrations obtained by rejection sampling on Qwen3-32B responses~\citep{huan2025transferability}.

\paragraph{Training setup.}
For all RL experiments, we train with GRPO~\citep{shao2024deepseekmath} using
verl~\citep{sheng2024hybridflow} for 200 steps at a learning rate of $10^{-6}$. Each step samples 512 prompts and generates 16 responses per prompt (up to 16,384 tokens), followed by one epoch of actor updates with a minibatch size of 128. For SFT, we train for 1.5 epochs using a batch size of 512 and a learning rate of $5\times 10^{-5}$. Further training details are given in \Cref{app:experiment-details}.

\paragraph{Evaluation protocol.}
We evaluate transfer using benchmarks of mathematics, coding, knowledge and planning,
and general capabilities.
Our default inference settings use temperature 0.6, top-$p$ 0.95, and a maximum generation length of 32,768 tokens, though specific benchmarks retain their original protocols such as greedy decoding. \Cref{app:evaluation-protocols} gives the benchmark list and complete task-specific sampling and scoring details.

\subsection{Cross-domain transfer and forgetting}
\label{sec:transfer-forgetting}

\input{arxiv/figures/transfer_forgetting_figure}

\Cref{fig:transfer-forgetting} illustrates the transfer and forgetting results. While RL consistently improves in-domain math performance, out-of-domain (OOD) behavior diverges significantly.
Incorporating coding data into training (Qwen3-Math-Code-RL) improves coding tasks but fails to yield universal transfer.
Moreover, despite identical training data, Qwen3-Math-RL and Gemma-3-Math-RL shift in opposite directions on many OOD tasks. 
SFT results exhibit similar model dependence: Qwen3-SFT shows selective transfer on knowledge and planning despite broad forgetting elsewhere, whereas Gemma-3-SFT suffers universal OOD degradation.

Together, these results indicate that training-domain success does not reliably predict OOD outcomes. Instead, \textbf{transferability emerges from the complex interaction among training data, evaluation benchmark, and base model}. Complete scores are reported in \Cref{tab:transfer-forgetting}. Although prior work suggests that RL can generalize better to OOD tasks than SFT \citep{huan2025transferability,chu2025sft}, our findings highlight that RL remains susceptible to forgetting, with the direction and magnitude of OOD performance changes varying substantially across settings \citep{chen2026retaining,shenfeld2026razor}.
% Our results additionally demonstrate that transfer performance varies across training and evaluation domains \citep{cheng2025crossdomain}, highlighting a complex interplay that standard accuracy metrics 
% alone cannot explain.

\subsection{Expansion and shrinkage of finite-budget coverage}
\label{sec:empirical-passn}
\input{arxiv/figures/qwen_gemma_passn_side_by_side_figure}

Task-level gains and losses show how training changes performance across domains. To examine how it changes the range of problems solved within each task, we compare coverage at different sampling budgets using Pass@$N$ (\Cref{fig:passn-qwen-gemma}).
For a fixed prompt, Pass@$N$ increases with Pass@1; however, benchmark-level comparisons do not necessarily preserve the Pass@1 ordering at large $N$.

\paragraph{Reduced coverage at large sampling budgets.}
The most prevalent pattern we observe is higher Pass@1 alongside lower Pass@1024 relative to the base model, especially in mathematics. This agrees with the coverage contrast reported by \citet{yue2504does}.

\paragraph{Broader coverage when base success is rare.}
On the Reasoning Gym task BoxNet~\citep{stojanovski2026reasoning}, where the base model's success is exceptionally rare, Qwen3-Math-RL raises Pass@1 from 3.8\% to 9.9\% and Pass@1024 from 34.0\% to 51.0\%. Gemma-3-Math-RL also raises Pass@1024 coverage, although its Pass@1 falls. Qwen's joint gains show that higher one-shot accuracy can accompany higher Pass@1024, consistent with \citet{liu2025prorl,wen2026reinforcement}.

Our observations are broadly consistent with prior work~\citep{liu2025prorl}: shifts in Pass@$N$ are largely dictated by the benchmark distribution and the base model's initial ability, with substantial gains generally emerging where initial coverage is low. We further find that the training data itself can alter this regime: Qwen3-Math-RL and Qwen3-Math-Code-RL move CodeContests coverage in opposite directions. Consequently, whether RL expands or shrinks finite-budget coverage is context-dependent, rather than a universal RL effect or a trend that can be inferred from Pass@1.

We examine the solution trajectories underlying these contrasts in \Cref{sec:reasoning-mode-analysis}.

%% file: arxiv/figures/transfer_forgetting_figure.tex
\begin{figure}[t]
\centering
\includegraphics[width=\linewidth]{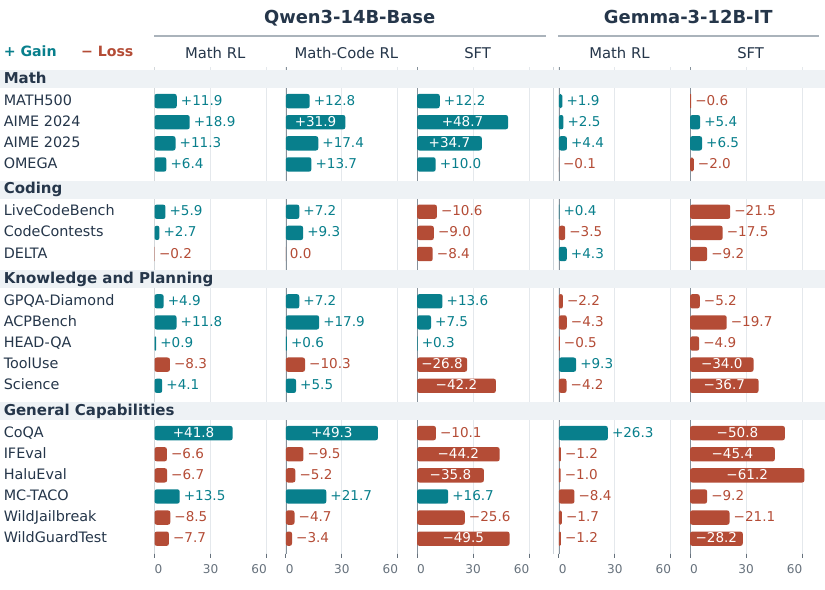}
\caption{\textbf{Transfer and forgetting depend on the benchmark, training data,
and base model.}
Each column reports changes in performance from the base model after training. RL consistently improves in-domain performance, but transfer behavior beyond the mathematics and coding domains varies significantly. In contrast, SFT exhibits catastrophic forgetting across a broad range of domains.}
\label{fig:transfer-forgetting}
\end{figure}

%% file: arxiv/figures/qwen_gemma_passn_side_by_side_figure.tex
\begin{figure}[t]
\centering
\includegraphics[width=\linewidth]
{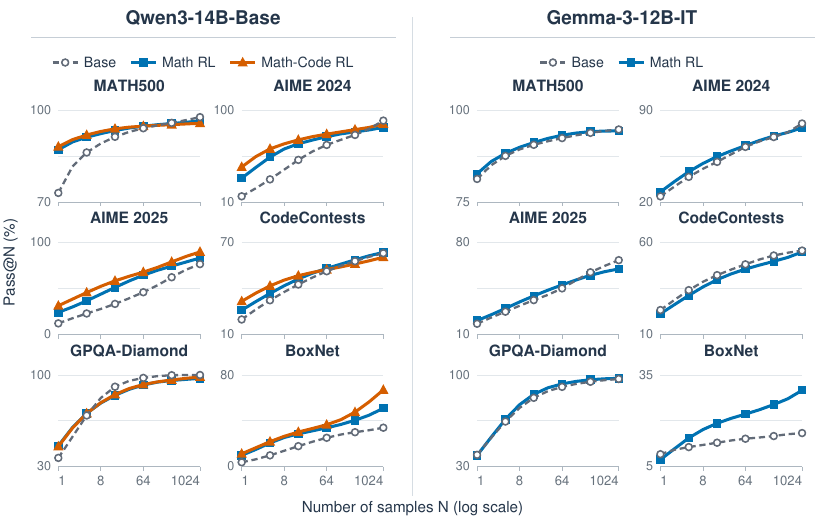}
\caption{\textbf{RL can either expand or contract reasoning coverage at large sampling budgets.} Higher Pass@1 can accompany either lower or higher Pass@1024 relative to the base model. The direction of the coverage change varies across training mixtures, benchmarks, and base models.}
\label{fig:passn-qwen-gemma}
\end{figure}

%% file: arxiv/sections/two_stage_model.tex
\section{From Empirical Observations to Reasoning Strategy Selection}
\label{sec:model}

To understand the performance and coverage changes in \Cref{sec:transfer-forgetting,sec:empirical-passn}, we compare solution trajectories before and after RL. The comparison indicates a shift in the strategies used to solve a question. We describe these strategies as \emph{reasoning routes} and examine how their use differs between the base and RL models. The findings motivate a model that separates strategy selection from conditional answer generation (\Cref{sec:two-stage}).

\subsection{Reasoning-route analysis}
\label{sec:reasoning-mode-analysis}

\subsubsection{Analysis setup}

We compare Qwen3-14B-Base with Qwen3-Math-RL on 240 questions from six benchmarks, with 16 responses per model and question. A \emph{reasoning route} is the strategy used to reach a response's final conclusion, when identifiable. Route identity reflects the problem representation and key mechanism, not notation, arithmetic, or final answers. For example, analytic reduction and exhaustive search are different routes. For each question, we combine both models' strategy descriptions extracted by an LLM from their reasoning traces and group them by route.
% , with model identity and correctness withheld.
We average the number of distinct routes over equally sized samples of identifiable responses from both models, using the smaller number of identifiable responses available for each question, and then average equally over questions.
Separately, we check which RL routes appear among the question's original 16 base responses. 
\Cref{app:reasoning-mode-analysis} gives classification and statistical details.

\subsubsection{Empirical findings}

\input{arxiv/figures/reasoning_routes_figure.tex}

\paragraph{RL concentrates on fewer routes per question.}
\Cref{fig:reasoning-routes}(a) shows lower mean route counts for Qwen3-Math-RL on all six benchmarks, with the largest decreases on coding benchmarks. Across questions, the mean falls from 3.24 for the base model to 2.66 for RL, with a paired difference (RL minus base) of $-0.58$ (95\% confidence interval $[-0.78,-0.38]$). For example, a LiveCodeBench question asks for the minimum absolute difference between array elements at least $x$ indices apart. Its 16 base responses use three routes, including comparison of neighboring values found by binary search, which all 16 RL responses use.

\paragraph{RL also uses routes absent from the base samples.}
Of 3,840 Qwen3-Math-RL responses, 487 (12.7\%) follow a route absent from the base responses to the same question (\Cref{fig:reasoning-routes}(b)). For example, on a LiveCodeBench question finding pattern occurrences within distance $k$ of another pattern's occurrence, all 16 base responses compare occurrence pairs directly, whereas 11 RL responses use binary search over occurrence lists.

% \noindent\textcolor{red}{\textbf{TODO 1.} Trim the route-analysis appendix to the trajectory-first framing.}

\subsection{A two-stage autoregressive model of chain-of-thought reasoning}
\label{sec:two-stage}

For analysis, we assume a shared library of reasoning strategies and represent a chain-of-thought trajectory as a strategy $z$ followed by an answer $y$. Both $z$ and $y$ may correspond to sequences of tokens rather than single tokens, yielding an autoregressive factorization at the block level. This abstraction separates strategy selection from execution, while the shared library provides a common space for studying transfer, coverage, and compute scaling in \Cref{sec:transfer,sec:passn,sec:scaling}.

Let $[M]=\{1,\ldots,M\}$ be a finite library of reasoning strategies and let
$\mathcal Y$ be a finite answer space.  For any prompt $x\in\mathcal{X}$, we consider a two-stage policy that first
generates a reasoning strategy $z\in[M]$ and then an answer $y\in\mathcal{Y}$ conditional on $z$. The joint policy is
\begin{equation}
\pi_{\theta,w}(z,y|x)
=\pi_\theta(z|x)\pi_w(y|x,z),
\label{eq:model}
\end{equation}
where both stages are linear softmax policies,
\begin{equation}
\pi_\theta(z|x)
=\frac{\exp(\theta_z^\top\phi(x))}
{\sum_{k=1}^M\exp(\theta_k^\top\phi(x))},
\qquad
\pi_w(y|x,z)
=\frac{\exp(w_y^\top\phi(x,z))}
{\sum_{v\in\mathcal Y}\exp(w_v^\top\phi(x,z))}.
\label{eq:softmax}
\end{equation}
Here $\theta_z,w_y\in\R^d$ are trainable parameters for strategy selection and answer generation, respectively. The feature maps $\phi(x)$ and $\phi(x,z)$ are frozen.
For a given prompt $x$, let $\mathcal M(x)\subseteq[M]$ be its non-empty set of
valid reasoning strategies and let $\mathcal A(x)\subseteq\mathcal Y$ be its non-empty
correct-answer set. We define the conditional and end-to-end scores as
\begin{equation}
R_w(x,z)
=\sum_{y\in\mathcal A(x)}\pi_w(y|x,z),
\qquad
R_{\theta,w}(x)
=\sum_{z\in\mathcal M(x)}
\pi_\theta(z|x)R_w(x,z).
\label{eq:reward}
\end{equation}

\subsection{Optimization dynamics}
Let $\{x_1,\ldots,x_n\}$ be the finite training set.
For $\beta>0$, define the following KL-regularized training objectives:
\begin{equation}
J_{\mathrm{strategy}}(\theta,w)
=\frac1n\sum_{i=1}^n\left[
R_{\theta,w}(x_i)
-\beta\KL\!\left(
\pi_\theta(\cdot|x_i)\,\middle\|\,
\pi_{\theta_0}(\cdot|x_i)
\right)\right],
\label{eq:mode-objective}
\end{equation}
\begin{equation}
J_{\mathrm{ans}}(\theta,w)
=\frac1n\sum_{i=1}^n
\Big[R_{\theta,w}(x_i)-\beta\mathbb E_{\pi_\theta(\cdot|x_i)}\big[
\mathbbm{1}_{z\in\mathcal M(x_i)}
\KL\!\big(
\pi_w(\cdot|x_i,z)\,\|\,
\pi_{w_0}(\cdot|x_i,z)
\big)
\big]\Big].
\label{eq:answer-objective}
\end{equation}
% Following the two-stage factorization, the strategy policy $\pi_\theta$ uses end-to-end success rate as its return, while the answer policy $\pi_w$ learns from visited valid strategies. Each KL term anchors the corresponding policy to the base policy.
Motivated by the small learning rates used in prominent RL pipelines
\citep{yu2025dapo,deepseekai2025r1}, we analyze the policy gradient flow given by
\begin{equation}
\dot\theta_t=\nabla_\theta J_{\mathrm{strategy}}(\theta_t,w_t),
\qquad
\dot w_t=\nabla_w J_{\mathrm{ans}}(\theta_t,w_t),
\qquad
(\theta_t,w_t)\big|_{t=0}=(\theta_0,w_0).
\label{eq:pg-flow}
\end{equation}
Along the RL dynamics, define $R_t(x,z):=R_{w_t}(x,z), R_t(x):=R_{\theta_t,w_t}(x)$.
Throughout our theoretical analysis, we consider a vanishing but strictly positive KL coefficient, $\beta\downarrow 0$. Without KL regularization, the policy collapses to the valid strategies with the largest initial probability.

%% file: arxiv/figures/reasoning_routes_figure.tex
\begin{figure}[!t]
\centering
\includegraphics[width=\linewidth]{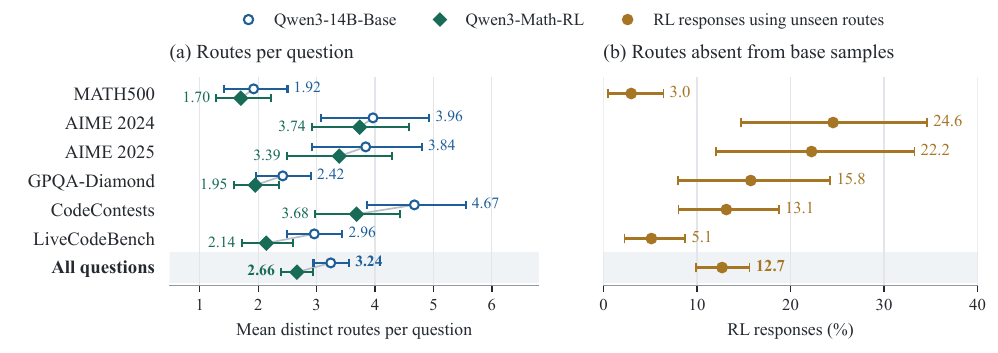}
\caption{\textbf{RL uses fewer routes, but also routes absent from base samples.}
(a) Mean numbers of distinct routes, using equal numbers of identifiable responses from both models for each question.
(b) Percentage of RL responses whose route is absent from the same question's 16 base responses.}
\label{fig:reasoning-routes}
\end{figure}

%% file: arxiv/sections/transfer_forgetting.tex
\section{Reasoning Transfer and Forgetting}
\label{sec:transfer}

We now use the two-stage model to analyze how RL-induced shifts in strategy selection transfer across tasks and influence target performance.

\begin{Assumption}[Realizability on training data]
\label{ass:source-realizability}
The training-prompt features are orthogonal:
\begin{equation}
\left\langle\phi(x_i),\phi(x_j)\right\rangle=0,\qquad \forall~i,j\in[n],\ i\ne j.
\label{eq:source-orthogonality}
\end{equation}
Moreover, for every $y\in\mathcal Y$, there exists $u_y$ such that
\begin{equation}
\phi(x_i,z)^\top u_y
=\mathbbm{1}_{y\in\mathcal A(x_i)}
-|\mathcal A(x_i)|/|\mathcal Y|,
\qquad \forall~i\in[n],\quad z\in\mathcal M(x_i).
\label{eq:source-realizability}
\end{equation}
\end{Assumption}
This is an overparameterized setting. We assume the feature orthogonality in
\Cref{eq:source-orthogonality} for clarity; the analysis extends to linearly
independent features. \Cref{eq:source-realizability} ensures linear separation
of the training data within the answer space.

\subsection{RL transferability from implicit bias}

For a target prompt $x'$, define its feature overlap with each
training prompt by
\begin{equation}
c(x',x_i)
=\left\langle\phi(x'),\phi(x_i)\right\rangle/\|\phi(x_i)\|^2.
\label{eq:transfer-coefficients}
\end{equation}
% Its sign and magnitude determine how the gate update learned at $x_i$ is transported to $x'$.
Define the strategies with largest aggregate training support by
\begin{equation}
\mathcal S(x')
=\operatorname*{arg\,max}_{z\in[M]}
\sum_{i=1}^n c(x',x_i)\mathbbm{1}_{z\in\mathcal M(x_i)}.
\label{eq:selected-modes}
\end{equation}
Let $\pi_\infty^\beta$ denote the convergent policy at fixed $\beta$ and
define the vanishing-KL limit $\pi_{\rm RL}:=\lim_{\beta\downarrow0}\pi_\infty^\beta$.
\begin{thm}[Transferability of RL]
\label{thm:pg-transfer}
Under \Cref{ass:source-realizability}, for any
finite $t$ and $z\in[M]$,
\begin{equation}
\pi_{\theta_t}(z|x')
\propto \pi_{\theta_0}(z|x')
\prod_{i=1}^n
\left[
\frac{\pi_{\theta_t}(z|x_i)}
{\pi_{\theta_0}(z|x_i)}
\right]^{c(x',x_i)}.
% \log\pi_{\theta_t}(z|x')
% = \log\pi_{\theta_0}(z|x')
% +\sum_{i=1}^n
% c(x',x_i)[\log
% \pi_{\theta_t}(z|x_i)-\log\pi_{\theta_0}(z|x_i)]+\rm{const}.
\label{eq:pg-transfer-identity}
\end{equation}
In the vanishing KL limit, the RL strategy policy is the base policy conditioned on
$\mathcal S(x')$,
% \begin{equation}
% \pi_{\rm RL}(z|x')=
% \frac{\pi_0(z|x')\mathbbm{1}_{z\in\mathcal S(x')}}
% {\pi_0(\mathcal S(x')|x')},
% \label{eq:pg-target-low-beta}
% \end{equation}
and the probability of selecting a valid strategy is
\begin{equation}
\pi_{\rm RL}(\mathcal M(x')|x')
=\frac{\pi_0(\mathcal S(x')\cap\mathcal M(x')|x')}
{\pi_0(\mathcal S(x')|x')}.
\label{eq:pg-target-correct-mass}
\end{equation}
\end{thm}

\begin{Assumption}[Target answer-feature orthogonality]
\label{ass:answer-orthogonality}
The answer features of $x'$ are orthogonal to source-valid answer features:
$
\left\langle\phi(x',z'),\phi(x_i,z)\right\rangle=0, \forall\ z\in\mathcal{M}(x_i),z'\in[M].
$
\end{Assumption}

\begin{coro}
\label{cor:pg-target-reward}
Under \Cref{ass:source-realizability,ass:answer-orthogonality}, in the
vanishing-KL limit the reward on the target prompt is
\begin{equation}
R_{\pi_{\rm RL}}(x')
=
\frac{
\sum_{z\in\mathcal S(x')\cap\mathcal M(x')}
\pi_0(z|x')R_0(x',z)}
{\pi_0(\mathcal S(x')|x')}.
\label{eq:pg-target-reward}
\end{equation}
\end{coro}

Detailed proofs are deferred to \Cref{app:transfer-proofs}.
\Cref{thm:pg-transfer} shows how changes in strategy preferences on training prompts transfer to target prompts through feature overlap. Under answer-feature orthogonality, \Cref{cor:pg-target-reward} translates this strategy selection into end-to-end reward.

Our results show that the inductive bias of RL preserves the base model's preference among the surviving strategies \citep{shenfeld2026razor}.
More crucially, this preservation does not determine target performance. \textbf{Feature overlap governs which strategy preferences transfer, but cannot alone predict whether the target task benefits}. If $\mathcal S(x')\cap\mathcal M(x')=\emptyset$, all valid strategies are suppressed, thereby inducing forgetting. Therefore, RL transferability is jointly determined by the reasoning strategies favored in the training data, the valid strategies required by the target task, and the base model's bias.
The heterogeneous shifts in \Cref{fig:transfer-forgetting} are compatible with these selective effects and match the qualitative predictions.
% Math-RL training improves ACPBench but degrades ToolUse for Qwen, yet causes the opposite effect for Gemma.
% The training mixture also matters: Qwen3-Math-Code-RL gains substantially more on coding domain than Qwen3-Math-RL. 
% Together, these observations match the theorem's qualitative prediction. 
Further prompt-level analyses are given in \Cref{app:transfer-output-analysis}.

\subsection{When does SFT exhibit catastrophic forgetting}

To compare with RL, we analyze SFT under the same theoretical setup.
For each training prompt $x_i$, we sample an expert demonstration
$z_{E,i}\sim\pi_E(\cdot|x_i),y_{E,i}\sim\pi_E(\cdot|x_i,z_{E,i})$, conditioning on
$z_{E,i}\in\mathcal M(x_i)$ and $y_{E,i}\in\mathcal A(x_i)$.
The training objective is
\begin{equation}
J_{\rm SFT}(\theta,w)
=\frac1n\sum_{i=1}^n
\left[
\log\pi_\theta(z_{E,i}|x_i)
+\log\pi_w(y_{E,i}|x_i,z_{E,i})
\right].
\label{eq:sft-objective}
\end{equation}
Consider the gradient flow from
$(\theta_0,w_0)$ on \Cref{eq:sft-objective} and let $\pi_{\rm SFT}$ be the convergent policy.
As in the RL analysis, define the strategies with largest support from
the realized demonstrations:
\begin{equation}
\mathcal S_E(x')
=\operatorname*{arg\,max}_{z\in[M]}
\sum_{i=1}^n c(x',x_i)\mathbbm{1}_{z=z_{E,i}}.
\label{eq:sft-selected-modes}
\end{equation}

\begin{thm}[Transferability of SFT]
\label{thm:sft-transfer}
Under \Cref{ass:source-realizability}, for any realized expert traces,
the convergent SFT strategy policy satisfies
\begin{equation}
\pi_{\rm SFT}(z|x')>0
\quad\Longleftrightarrow\quad z\in\mathcal S_E(x').
\label{eq:sft-target-support}
\end{equation}
The probability of valid strategies tends to one if
$\mathcal S_E(x')\subseteq\mathcal M(x')$ and to zero if
$\mathcal S_E(x')\cap\mathcal M(x')=\varnothing$.
\end{thm}

The definition of supporting strategies in \Cref{eq:selected-modes,eq:sft-selected-modes} exposes a fundamental distinction: RL amplifies all source-valid strategies, whereas SFT reinforces only the demonstrated strategies.
Consequently, \textbf{the RL policy broadly preserves the base model's prior, whereas SFT behavior depends heavily on the characteristics of the teacher policy}. SFT is therefore more prone to catastrophic forgetting, as observed in \Cref{fig:transfer-forgetting}. However, this vulnerability is not absolute; Qwen3-SFT, for instance, outperforms the RL models on the OOD benchmark GPQA-Diamond.

More broadly, the transferability of SFT is highly sensitive to compatibility between the student and teacher policies. Since the SFT demonstrations originate from Qwen3-32B, they naturally align better with Qwen3-14B-Base than with Gemma-3-12B-IT, accounting for the substantially broader forgetting of Gemma-3-SFT in \Cref{fig:transfer-forgetting}. Further examples and analysis are provided
in \Cref{app:transfer-output-analysis}.

%% file: arxiv/sections/passn_boundary.tex
\section{Expansion and Shrinkage of Finite-Budget Coverage}
\label{sec:passn}
We now examine how these prompt-level shifts in strategy preference aggregate across an evaluation benchmark $\mathcal{B}$ to alter overall coverage.
% For any policy $\pi$, let $R_\pi(x')$ denote its success rate. 
Define the Pass@$N$ metric as
\begin{equation}
\operatorname{Pass@}N(\pi)
=\E_{x'\sim\mathcal B}\left[1-\bigl(1-R_\pi(x')\bigr)^N\right].
\end{equation}
To isolate strategy selection, we evaluate each strategy policy with an ideal answer decoder: a response succeeds exactly when its strategy is valid. Consequently, \Cref{thm:pg-transfer} gives:
\begin{equation}
R_{\pi_0}(x')=\pi_0(\mathcal M(x')|x'),
\qquad
R_{\pi_{\rm RL}}(x')
=\frac{\pi_0(\mathcal S(x')\cap\mathcal M(x')|x')}
{\pi_0(\mathcal S(x')|x')}.
\end{equation}

\begin{thm}
\label{thm:passn-regimes}
Under \Cref{ass:source-realizability}, fix $\alpha\in(0,1]$ and an integer $N>\log(4/\alpha)$.
\begin{enumerate}[label=(\Roman*), leftmargin=*]
    \item \textbf{Strategy loss (higher base-model coverage):}
    If
    \begin{equation*}
    \Pr_{x'\sim\mathcal B}\left[
    \mathcal S(x')\cap\mathcal M(x')=\emptyset
    \right]\geq\alpha,
    \quad
    \Pr_{x'\sim\mathcal B}\left[
    \pi_0(\mathcal M(x')|x')
    \leq\frac{\log(4/\alpha)}{N}
    \right]\leq\frac{\alpha}{4},
    \end{equation*}
    then
    \begin{equation*}
    \operatorname{Pass@}N(\pi_0)
    -\operatorname{Pass@}N(\pi_{\mathrm{RL}})
    >\frac{\alpha}{2}.
    \end{equation*}

    \item \textbf{Rare-strategy amplification (higher RL coverage):} If
    \begin{equation*}
    \Pr_{x'\sim\mathcal B}\left[
    \pi_0(\mathcal M(x')|x')\leq\frac{1}{4N}
    \right]\geq\alpha,
    \quad 
    \Pr_{x'\sim\mathcal B}\left[
    \frac{\pi_0(\mathcal S(x')\cap\mathcal M(x')|x')}
    {\pi_0(\mathcal S(x')|x')}
    \leq\frac{\log(4/\alpha)}{N}
    \right]\leq\frac{\alpha}{4},
    \end{equation*}
    then
    \begin{equation*}
    \operatorname{Pass@}N(\pi_{\mathrm{RL}})
    -\operatorname{Pass@}N(\pi_0)
    >\frac{\alpha}{4}.
    \end{equation*}
\end{enumerate}
\end{thm}

Detailed proofs are provided in
\Cref{app:passn}.
The theorem shows that Pass@$N$ is governed by the low-success tail, particularly prompts whose one-shot success rate is of order $1/N$~\citep{chen2025coverage}, rather than by average Pass@1.
% In case~\textnormal{(I)}, the RL policy selects no valid strategy on an $\alpha$ fraction of prompts, whereas the base policy leaves few questions below the $1/N$ threshold.
% Conversely, in case~\textnormal{(II)}, the base policy leaves a substantial fraction of the benchmark below this scale, whereas the RL policy amplifies these rare strategies, thereby expanding coverage.
Coverage decreases when RL suppresses strategies on such prompts and increases when it amplifies rare successes above this scale. Both directions arise from reweighting an existing strategy library, so broader sampled coverage need not imply new strategies.

Consistent with our transferability results, shifts in finite-budget coverage depend jointly on the \textbf{training data}, \textbf{benchmark distribution} and \textbf{base model}.
\Cref{fig:passn-qwen-gemma} illustrates this interaction empirically: Qwen3-Math-RL expands coverage on BoxNet but shrinks it on AIME 2024 despite Pass@1 gains.
Altering the training mixture with coding data completely inverts Qwen's performance on CodeContests, while Math RL yields opposite AIME 2025 coverage shifts for Qwen and Gemma.
More prompt-level analysis and evaluation caveats are in \Cref{app:passn-question-analysis}.

%% file: arxiv/sections/scaling_laws.tex
\section{Scaling Laws of RL Compute}
\label{sec:scaling}

Beyond transfer and reasoning coverage, recent studies ask how performance
scales with RL compute.
\citet{khatri2025scalerl} introduce a \textbf{log-sigmoid} law:
\begin{equation}
R(C)=R_0+\frac{A-R_0}{1+(C_{\rm mid}/C)^B}.
% R(C)=R_0+(A-R_0)\sigma(B\log(C/C_{\rm mid})).
\label{eq:empirical-log-sigmoid}
\end{equation}
% where $A$, $B$, and $C_{\rm mid}$ represent the log-sigmoid asymptote, exponent, and midpoint, respectively.
However, the sigmoidal plateau typically emerges only during a late saturation stage, far beyond the reach of most compute-constrained RL training. Alternatively, several works \citep{meta2026musespark,tan2026scaling,shen2026understanding} apply a first-order approximation to derive a \textbf{log-linear} law:
\begin{equation}
R(C)=R_0+B_{\rm lin}\log(C/C_{\rm ref}).
\label{eq:empirical-log-linear}
\end{equation}
We examine these empirical laws and provide a rigorous theoretical account
of RL scaling behavior based on strategy-selection dynamics.

\subsection{Empirical scaling laws in RL}
\label{sec:empirical-scaling}

\input{arxiv/figures/qwen3_14b_pre200_gpu_hour_extrapolation_figure}

We fit \Cref{eq:empirical-log-sigmoid,eq:empirical-log-linear} to
Qwen3-Math-RL and Qwen3-Math-Code-RL evaluations through step 200, using
cumulative GPU-hours as compute. 
We extend both RL runs to step 350 to test predictions beyond the range of
compute used for fitting, where the cumulative compute is nearly $2\times$ the fitting budget.
\Cref{fig:qwen-scaling-extrapolation-gpu-hours} shows the observed scores and
fitted curves during fitting and extrapolation. Detailed fitting procedures,
coefficient estimates, and results using other compute metrics are given in
\Cref{app:scaling-fit-procedure}.

As shown in \Cref{fig:qwen-scaling-extrapolation-gpu-hours}, the log-sigmoid
and log-linear fits are often nearly indistinguishable over the range of
compute used for fitting. This is expected because the log-linear law in
\cref{eq:empirical-log-linear} locally approximates the log-sigmoid law in
\cref{eq:empirical-log-sigmoid}.
The laws diverge only during extrapolation: the log-linear law maintains a
constant rate of improvement with log compute and is therefore unbounded, whereas
the log-sigmoid law bends toward a finite plateau.

During the additional training, Qwen3-Math-RL follows the log-linear
prediction more closely, yet Qwen3-Math-Code-RL exhibits saturation: the
log-sigmoid curve better tracks late-stage performance on MATH500, AIME, and
ACPBench, as well as the average scores. Meanwhile, the log-linear model
increasingly overpredicts performance at higher compute budgets.

The reliability of scaling-law predictions also varies by evaluation domain.
In-domain mathematical benchmarks consistently follow the fitted predictions,
whereas extrapolation to coding and planning benchmarks is more brittle.
Individual benchmark trajectories are also less stable than their average,
reflecting the noisy dynamics of practical RL training.
Notably, while \citet{khatri2025scalerl,tan2026scaling} also report extrapolative scaling laws, their evaluations are restricted to held-out data drawn from the training distribution. In contrast, our analysis incorporates out-of-distribution and cross-domain evaluations. Our setting therefore tests not only whether performance scales predictably with RL compute, but also whether this scaling behavior transfers to new domains.

\subsection{Theoretical justification for scaling laws}

We now investigate why these simple laws capture RL dynamics via reasoning-strategy selection.
For ease of exposition, we take $n=1$, i.e., there is only one training prompt denoted by $x$.
For an evaluation benchmark $\mathcal B$, define its reward during training
as $R(t):=\E_{x'\sim\mathcal B}[R_t(x')]$.

\begin{Assumption}[Strategy-dominated regime]
\label{ass:scaling-balance}
There exists $T_0$ such that $\pi_0(\cdot|x)$ and
$\pi_{T_0}(\cdot|x)$ are uniform within $\mathcal M(x)$ and within
$[M]\backslash\mathcal M(x)$. From $T_0$ onward, the answer policy saturates and remains frozen: $R_t(x,z)=R_\infty^\beta(x,z)=\rho_\beta$ for all $t\geq T_0,z\in\mathcal{M}(x)$.
\end{Assumption}
After $T_0$, training reduces to the evolution of the probability of valid
strategies. We use elapsed time $C=t-T_0$ as a proxy for compute and fit the
scaling laws over an interval $(a,b)$ of this probability. For any
$u\in[a,b]$, let $C_u$ be the first time such that
$\pi_{T_0+C_u}(\mathcal M(x)|x)=u$.
For any $x'\sim\mathcal B$, we assume feature overlap $c(x',x)>0$ and positive transfer.

% \paragraph{Log-sigmoid transfer.}
% % Let $\mathsf{logit}(u):=\log(u/(1-u))$,
% % $H(u):=(2u-1)/[u(1-u)]+2\mathsf{logit}(u)$, and
% Let $\mathsf{logit}(u):=\log(u/(1-u))$,
% $H(u):=\frac{2u-1}{u(1-u)}+2\mathsf{logit}(u)$, and
% \begin{equation}
% \delta_{a,b}:=\mathsf{logit}(b)-\mathsf{logit}(a),\qquad
% \Gamma_{a,b}^{(p)}:=\frac12\inf_{\lambda>0}
% \operatorname{osc}_{u\in[a,b]}
% \left\{\mathsf{logit}(u)-\lambda\log[H(u)-H(p)]\right\},
% \label{eq:benchmark-scaling-window}
% \end{equation}
% where $\operatorname{osc}$ denotes the difference between the maximum and minimum values.
% \begin{thm}
% \label{thm:benchmark-sigmoid-scaling}
% Let $c_{\max}:=\sup_{x'\sim \mathcal{B}}c(x',x)<\infty$. Under \Cref{ass:source-realizability,ass:answer-orthogonality,ass:scaling-balance},
% for any $\varepsilon_{\rm fit}>0$ and sufficiently
% small $\beta$, there exist $R_0,A,B,C_{\rm mid}>0$ such that
% \begin{equation}
% \sup_{C\in[C_a,C_b]}
% \left|\frac{R(T_0+C)-R_0}{A-R_0}
% -\frac{1}{1+(C_{\rm mid}/C)^B}\right|
% \leq\tanh\left(\frac{c_{\max}^2\delta_{a,b}^2}{256}
% +\frac{c_{\max}\Gamma_{a,b}^{(p)}}4\right)+\varepsilon_{\rm fit}.
% \label{eq:benchmark-sigmoid-main}
% \end{equation}
% \end{thm}

\begin{thm}[Informal]
\label{thm:benchmark-sigmoid-scaling}
Under \Cref{ass:source-realizability,ass:answer-orthogonality,ass:scaling-balance}, there exist $R_0,A,B,C_{\rm mid}$ such that
\begin{equation}
\sup_{C\in[C_a,C_b]}
\left|R(T_0+C)-R_0
-\frac{A-R_0}{1+(C_{\rm mid}/C)^B}\right|
\leq \mathcal{O}((b-a)^2).
\label{eq:benchmark-sigmoid-main}
\end{equation}
\end{thm}
% This theorem shows that the log-sigmoid law can approximate the RL dynamics uniformly up to an error depending on the interval $(a,b)$ and initial accuracy $p$. Here we adopt a uniform upper bound on feature overlap; \Cref{thm:benchmark-bounded-overlap-scaling} gives a sharper distribution-dependent bound.

% \paragraph{Log-linear transfer.}
% Define the reward-gain-weighted feature overlap and the log-compute span:
% \begin{equation}
% \bar c:=\frac{\E_{x'\sim\mathcal B}\left[c(x',x)
% \bigl(R_{T_0+C_b}(x')-R_{T_0+C_a}(x')\bigr)\right]}
% {R(T_0+C_b)-R(T_0+C_a)},\qquad
% \Delta_{a,b}^{(p)}:=\log\frac{H(b)-H(p)}{H(a)-H(p)}.
% \label{eq:benchmark-mean-overlap}
% \end{equation}
% \begin{thm}
% \label{thm:benchmark-linear-scaling}
% Under the same assumptions,
% there exist $R_0,B_{\rm lin},C_{\rm ref}>0$ such that
% \begin{equation}
% \sup_{C\in[C_a,C_b]}
% \frac{\left|R(T_0+C)-R_0-B_{\rm lin}\log(C/C_{\rm ref})\right|}
% {R(T_0+C_b)-R(T_0+C_a)}
% \leq\frac{(1+\bar c)\delta_{a,b}+\Delta_{a,b}^{(p)}}{16}
% +\varepsilon_{\rm fit}.
% \label{eq:benchmark-linear-scaling}
% \end{equation}
% \end{thm}
\begin{thm}[Informal]
\label{thm:benchmark-linear-scaling}
Under the same assumptions,
there exist $R_0,B_{\rm lin},C_{\rm ref}$ such that
\begin{equation}
\sup_{C\in[C_a,C_b]}
\left|R(T_0+C)-R_0-B_{\rm lin}\log(C/C_{\rm ref})\right|
\leq \mathcal{O}((b-a)^2).
\label{eq:benchmark-linear-scaling}
\end{equation}
\end{thm}
The results show that both log-sigmoid and log-linear laws closely approximate RL dynamics over a local compute range.
Formal statements and error bounds are deferred to \Cref{thm:benchmark-bounded-overlap-scaling,thm:benchmark-linear-scaling-formal} in the appendix.
Direct computation verifies that our error bounds remain below 1\% for fitting intervals spanning up to 30\%, as visualized in \Cref{app:scaling-bound-magnitudes}.
\Cref{app:scaling-fit-interpretation} examines the empirical implications of this local approximation and the interpretation of fitted scaling coefficients.

% The fitted slopes and
% saturation scales consequently depend on both the training mixture and the
% evaluation distribution, consistent with target-specific feature overlap in
% our theory.

%% file: arxiv/figures/qwen3_14b_pre200_gpu_hour_extrapolation_figure.tex
\begin{figure}[t]
\centering
\includegraphics[width=\linewidth]
{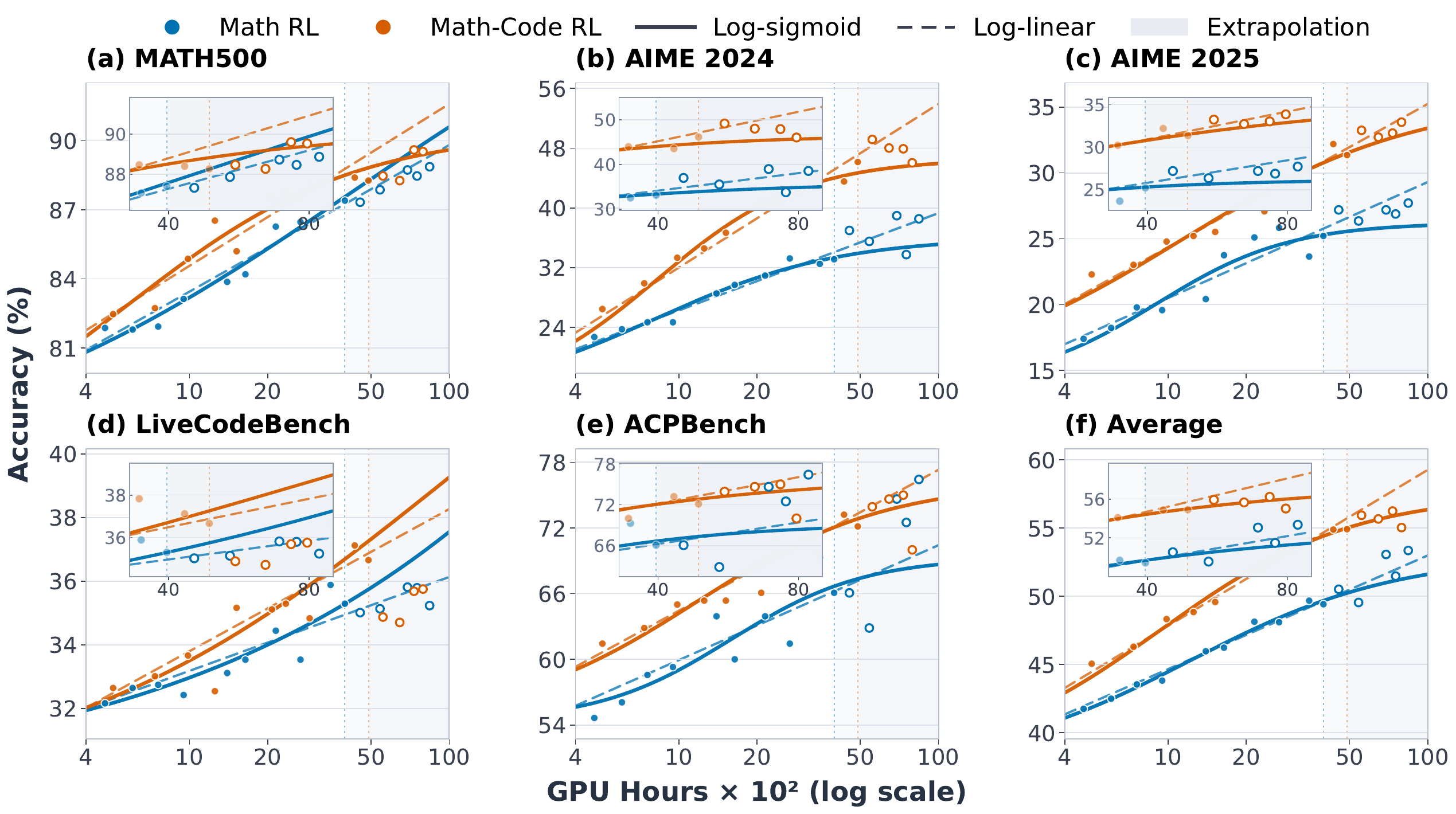}
\caption{\textbf{Scaling behavior of RL compute.} Both laws capture the fitting
regime effectively with $\mathsf{R^2}>0.99$ on average.
Extrapolative predictions are highly reliable up to 2$\times$ compute in the math domain and average scores, yet exhibit greater variation on coding and planning tasks due to noisy RL training dynamics.}
\label{fig:qwen-scaling-extrapolation-gpu-hours}
\end{figure}

%% file: arxiv/sections/related_work.tex
\section{Related Work}
\label{sec:related-work}

\paragraph{Reinforcement learning for LLM reasoning and chain of thought.}
CoT prompting and self-consistency establish the value of generating and
sampling reasoning trajectories~\citep{wei2022chain,wang2023selfconsistency}.
Reinforcement learning with verifiable rewards (RLVR) trains models using
automatically checked task outcomes; T\"ulu 3 incorporates this approach into
an open post-training recipe~\citep{lambert2024tulu}.
\citet{shao2024deepseekmath} introduces GRPO, and DeepSeek-R1 demonstrates large-scale reasoning gains from RL~\citep{deepseekai2025r1}.
DAPO and Open-Reasoner-Zero further develop open systems for scaling RL
training on reasoning tasks~\citep{yu2025dapo,hu2025openreasonerzero}.
\citet{liu2025understanding} examine how base-model capabilities and biases
in GRPO affect training, and introduce Dr. GRPO to improve token efficiency.

\paragraph{Transferability, forgetting, and the reasoning boundary.}
\citet{chu2025sft} find that RL
trained with outcome rewards generalizes better than SFT to unseen textual
rules and visual variants in arithmetic and navigation tasks.
\citet{huan2025transferability} find that RL transfers more broadly and retains
performance better than SFT in controlled experiments, while
\citet{cheng2025crossdomain} show that transfer varies with training domain
and task difficulty.
\citet{chen2026retaining} identify on-policy data as a major factor in RL's
relative resistance to forgetting. Coverage findings also differ:
\citet{yue2504does} observe gains at small sampling budgets but reduced
coverage at large ones, whereas \citet{liu2025prorl} report expansion under
prolonged RL training. \citet{wen2026reinforcement} further show that evaluating
reasoning correctness, rather than final answers alone, can change coverage
comparisons. Our framework characterizes conditions for these contrasting
outcomes within a shared model of strategy selection, distinguishing
coverage at a given sampling budget from the creation or loss of underlying strategies.

\paragraph{Scaling laws in reinforcement learning.}
\citet{khatri2025scalerl} fit log-sigmoid curves relating compute to performance and
validate extrapolation to longer RL runs. Meta's Muse Spark reports log-linear
gains in training scores with increasing RL steps~\citep{meta2026musespark}.
\citet{tan2026scaling} study joint scaling with model size, data, and compute
in mathematical reasoning, while \citet{shen2026understanding} relate the returns
from increasing RL compute to pretraining. These studies establish empirical
scaling relationships, but do not derive the commonly used log-sigmoid and
log-linear forms from RL training dynamics. Our work provides a theoretical
justification for both forms by deriving them as approximations to the
dynamics of reasoning strategy selection, with explicit error bounds.
We also examine their predictive reliability across evaluation domains.

\paragraph{Theoretical analyses of RL in LLMs.}
\citet{chen2026reshaping} analyze the selection of reasoning patterns in a two-stage tabular policy with fixed conditional success rates, focusing on behavior on training prompts. Complementary analyses
study how gradient gaps affect optimization stability~\citep{suk2025optimization},
success amplification under GRPO~\citep{mroueh2025grpo}, and pretrained response
coverage~\citep{chen2025coverage}. Our framework learns both strategy and answer
policies and uses shared prompt features to derive how source updates reweight
target strategies. This mechanism connects strategy selection to
transfer and forgetting, Pass@$N$ coverage, and scaling laws that describe
performance over a limited range of training compute.

%% file: arxiv/sections/conclusion.tex
\section{Conclusion and Discussion}
In this paper, we investigate how RL reshapes LLM reasoning. Starting from
empirical observations across model families and task domains, we analyze
solution traces before and after RL, showing a shift in the distribution
over reasoning strategies. We propose a two-stage policy model to
formalize this phenomenon and prove the conditions under which RL can transfer
or forget capabilities and expand or contract finite-budget coverage. We also
use this model to give a theoretical justification for empirical scaling laws
in RL compute. These results connect changes in strategy selection to
cross-domain transfer, reasoning coverage, and compute scaling.

%% file: arxiv/sections/ack.tex
\section*{Acknowledgements}

Y.H. and S.S. were supported by the U.S. Army Research Laboratory and the U.S. Army Research Office under Grant W911NF2010219, Office of Naval Research, and NSF. This work used Jetstream2 at Indiana University through allocation CIS240832 from the Advanced Cyberinfrastructure Coordination Ecosystem: Services \& Support (ACCESS) program, which is supported by National Science Foundation grants \#2138259, \#2138286, \#2138307, \#2137603, and \#2138296.

%% file: arxiv/appendix/transfer_forgetting_proofs.tex
\section{Proofs for Transfer and Forgetting}
\label{app:transfer-proofs}

\subsection{RL dynamics on training prompts}

The following auxiliary result establishes the limiting policy on training
prompts used to prove \Cref{thm:pg-transfer}, along with a convergence rate
used in the scaling analysis.

\begin{thm}[RL limit on training prompts]
\label{thm:pg-training}
Under \Cref{ass:source-realizability}, fix $\beta>0$.  RL converges
on every training prompt to a policy $\pi_\infty^\beta$. For each
$i\in[n]$, $z\in\mathcal M(x_i)$, and $y\in\mathcal Y$,
\begin{equation}
\pi_\infty^\beta(y|x_i,z)
=\frac{\pi_0(y|x_i,z)
e^{\mathbbm{1}_{y\in\mathcal A(x_i)}/\beta}}
{\sum_{v\in\mathcal Y}\pi_0(v|x_i,z)
e^{\mathbbm{1}_{v\in\mathcal A(x_i)}/\beta}},
\label{eq:pg-answer-limit}
\end{equation}
and hence
\begin{equation}
R_\infty^\beta(x_i,z)
=\frac{R_0(x_i,z)}
{R_0(x_i,z)+[1-R_0(x_i,z)]e^{-1/\beta}}.
\label{eq:pg-reward-limit}
\end{equation}
The limiting strategy policy is
\begin{equation}
\pi_\infty^\beta(z|x_i)\propto \pi_0(z|x_i)
\begin{cases}
e^{R_\infty^\beta(x_i,z)/\beta},&z\in\mathcal M(x_i),\\
1,&z\notin\mathcal M(x_i).
\end{cases}
\label{eq:pg-mode-limit}
\end{equation}
Moreover, there are constants $C_\beta,\lambda_\beta>0$ such that
\begin{align}
\max_i\left\|\pi_{\theta_t}(\cdot|x_i)
-\pi_\infty^\beta(\cdot|x_i)\right\|_{\rm TV}
&\le C_\beta e^{-\lambda_\beta t},
\nonumber\\
\max_{\substack{i\in[n]\\z\in\mathcal M(x_i)}}
\left\|\pi_{w_t}(\cdot|x_i,z)
-\pi_\infty^\beta(\cdot|x_i,z)\right\|_{\rm TV}
&\le C_\beta e^{-\lambda_\beta t}.
\label{eq:pg-training-rate}
\end{align}
\end{thm}

\begin{proof}
\Cref{ass:source-realizability} yields parameters
$(\theta_\star^\beta,w_\star^\beta)$ whose increments from
$(\theta_0,w_0)$ lie in the spans of the corresponding feature vectors for
the training prompts and satisfy
\begin{align}
\phi(x_i)^\top(\theta_{\star,z}^\beta-\theta_{0,z})
&=\frac1\beta\left[
\mathbbm{1}_{z\in\mathcal M(x_i)}R_\infty^\beta(x_i,z)
-\frac1M\sum_{k\in\mathcal M(x_i)}R_\infty^\beta(x_i,k)
\right],
\label{eq:target-gate-logits}\\
\phi(x_i,z)^\top(w_{\star,y}^\beta-w_{0,y})
&=\frac1\beta\left[
\mathbbm{1}_{y\in\mathcal A(x_i)}
-\frac{|\mathcal A(x_i)|}{|\mathcal Y|}
\right].
\label{eq:target-answer-logits}
\end{align}
The first identity holds for all $i,z$; the second holds for all
$i$, $z\in\mathcal M(x_i)$, and $y$.  In the first identity, the indicator
sets the strategy-level return to zero when $z\notin\mathcal M(x_i)$.  The
pairwise logit differences in \Cref{eq:target-gate-logits,eq:target-answer-logits}
yield \Cref{eq:pg-answer-limit,eq:pg-mode-limit}.
Choose the minimum-norm parameter increments in these spans.  Because the
right-hand sides are centered across softmax classes, the increments sum
to zero across classes, matching the gradient-flow gauge.

We prove convergence to these parameters within the feature spans generated by
the training prompts.
Set
\begin{equation}
E_\theta(t)=\frac12\sum_{z=1}^M
\|\theta_{t,z}-\theta_{\star,z}^\beta\|^2,
\qquad
E_w(t)=\frac12\sum_{y\in\mathcal Y}
\|w_{t,y}-w_{\star,y}^\beta\|^2.
\label{eq:source-energy}
\end{equation}
For a distribution $P$ and vector $a$, write
$\operatorname{Var}_P(a)=\sum_jP(j)(a_j-\sum_kP(k)a_k)^2$.
Direct differentiation of \Cref{eq:mode-objective} and
\Cref{eq:answer-objective}, using the target log-odds above, yields
\begin{align}
\dot E_w
&=-\frac\beta n\sum_{i=1}^n\sum_{z\in\mathcal M(x_i)}
\pi_{\theta_t}(z|x_i)
\operatorname{Var}_{\pi_{w_t}(\cdot|x_i,z)}
\!\left(
\phi(x_i,z)^\top(w_t-w_\star^\beta)
\right),
\label{eq:answer-dissipation}\\
\dot E_\theta
&=-\frac\beta n\sum_{i=1}^n
\operatorname{Var}_{\pi_{\theta_t}(\cdot|x_i)}
\!\left(
\phi(x_i)^\top(\theta_t-\theta_\star^\beta)
\right)
\nonumber\\
&\quad+\frac1n\sum_{i=1}^n
\operatorname{Cov}_{\pi_{\theta_t}(\cdot|x_i)}
\!\left(
\phi(x_i)^\top(\theta_t-\theta_\star^\beta),
\mathbbm{1}_{z\in\mathcal M(x_i)}
[R_t(x_i,z)-R_\infty^\beta(x_i,z)]
\right).
\label{eq:gate-dissipation}
\end{align}
The entries in each variance or covariance are indexed by the relevant
strategy or answer variable.

\Cref{eq:answer-dissipation} makes $E_w$ nonincreasing, so all
answer probabilities on the training prompts have a common lower bound
$\underline q>0$.
For fixed $i$ and $z\in\mathcal M(x_i)$, let
$a_y=\phi(x_i,z)^\top(w_{t,y}-w_{\star,y}^\beta)$ and interpolate between the
current and limiting answer policies via
$q_s(y)\propto\pi_{w_t}(y|x_i,z)e^{-s a_y}$, $0\le s\le1$.
Then $q_1=\pi_\infty^\beta(\cdot|x_i,z)$ and
$\frac{d}{ds}q_s(\mathcal A(x_i))
=-\operatorname{Cov}_{q_s}(\mathbbm{1}_{y\in\mathcal A(x_i)},a_y)$.
Taking the $q_0$-mean of $a$ as a candidate center and using
$q_s(y)\le1\le q_0(y)/\underline q$ gives
$\operatorname{Var}_{q_s}(a)\le
\underline q^{-1}\operatorname{Var}_{q_0}(a)$.
Since
$\operatorname{Var}_{q_s}(\mathbbm{1}_{y\in\mathcal A(x_i)})\le1/4$,
integration gives
\begin{equation}
|R_t(x_i,z)-R_\infty^\beta(x_i,z)|^2
\le\frac{1}{4\underline q}
\operatorname{Var}_{\pi_{w_t}(\cdot|x_i,z)}(a).
\end{equation}
Cauchy--Schwarz and Young's inequality now absorb the covariance in
\Cref{eq:gate-dissipation} into the two variance terms, giving
constants $\gamma,\eta>0$ such that
\begin{align}
\frac{d}{dt}(E_\theta+\gamma E_w)
\le-\frac\eta n\sum_{i=1}^n\Bigg[&
\operatorname{Var}_{\pi_{\theta_t}(\cdot|x_i)}
\!\left(\phi(x_i)^\top(\theta_t-\theta_\star^\beta)\right)
\nonumber\\
&+\sum_{z\in\mathcal M(x_i)}\pi_{\theta_t}(z|x_i)
\operatorname{Var}_{\pi_{w_t}(\cdot|x_i,z)}
\!\left(\phi(x_i,z)^\top(w_t-w_\star^\beta)\right)
\Bigg].
\label{eq:combined-dissipation}
\end{align}
The two parameter blocks are therefore bounded, which also gives a uniform
positive lower bound $\underline p>0$ for every strategy probability on a
training prompt. The parameter errors remain in these feature spans and sum
to zero across softmax classes. Consequently, each strategy-selection variance in
\Cref{eq:combined-dissipation} is at least $\underline p$ times the sum of
squared strategy-selection logit errors over strategies; each weighted answer variance is at
least $\underline p\underline q$ times the sum of squared answer-logit
errors over answers. After summing over training prompts and valid strategies,
these variance sums bound $E_\theta$ and $E_w$, respectively, because the
prompt features are orthogonal and the answer-feature Gram operator is positive
on its span.
Thus the right-hand side of
\Cref{eq:combined-dissipation} is at most
$-2\lambda_\beta(E_\theta+\gamma E_w)$ for some
$\lambda_\beta>0$.  Gronwall's inequality and smoothness of softmax then give
global existence and \Cref{eq:pg-training-rate}.

Summing \Cref{eq:pg-answer-limit} over $\mathcal A(x_i)$ gives
\Cref{eq:pg-reward-limit}.  Uniformly over the finite training set,
$1-R_\infty^\beta(x_i,z)=O(e^{-1/\beta})=o(\beta)$.  This uniform estimate
is used in the target limit below.
\end{proof}

\subsection{Proofs for RL transfer}

For fixed $\beta>0$, the target strategy policy converges to
\begin{equation}
\pi_\infty^\beta(z|x')
\propto \pi_0(z|x')
\exp\!\left\{
\frac1\beta\sum_{i:\,z\in\mathcal M(x_i)}
c(x',x_i)R_\infty^\beta(x_i,z)
\right\},
\label{eq:pg-target-limit}
\end{equation}
Moreover, after increasing $C_\beta$ in \Cref{eq:pg-training-rate} if needed,
the finite-time target policy satisfies
\begin{equation}
\left\|\pi_{\theta_t}(\cdot|x')
-\pi_\infty^\beta(\cdot|x')\right\|_{\rm TV}
\le
\left(\sum_{i=1}^n|c(x',x_i)|\right)
C_\beta e^{-\lambda_\beta t}.
\label{eq:pg-target-rate}
\end{equation}
If every $c(x',x_i)$ is zero, the target strategy policy remains at its
initial value for all time.

\begin{proof}[Proof of \Cref{thm:pg-transfer}]
For each strategy $z$, the RL dynamics keep
$\theta_{t,z}-\theta_{0,z}$ in
$\mathrm{span}\{\phi(x_1),\ldots,\phi(x_n)\}$.  Orthogonality and
\Cref{eq:transfer-coefficients} give
\begin{equation}
\phi(x')^\top(\theta_{t,z}-\theta_{0,z})
=\sum_{i=1}^n c(x',x_i)
\phi(x_i)^\top(\theta_{t,z}-\theta_{0,z}).
\label{eq:projected-logit-change}
\end{equation}
Subtracting \Cref{eq:projected-logit-change} for two strategies converts it into
an identity between their log-odds changes.  Exponentiating and normalizing
over $z$ proves \Cref{eq:pg-transfer-identity}.  Letting $t\to\infty$ and
using \Cref{eq:pg-mode-limit} proves \Cref{eq:pg-target-limit}.

The probability lower bounds established above and \Cref{eq:pg-training-rate}
bound every log-odds error on a training prompt by
$C_\beta e^{-\lambda_\beta t}$ after
increasing $C_\beta$ if needed.
Applying \Cref{eq:projected-logit-change} to the finite-time and limiting
parameters, summing the absolute values of the overlaps, and using Lipschitz continuity of softmax
gives \Cref{eq:pg-target-rate}.  If every overlap vanishes,
\Cref{eq:projected-logit-change} is zero for all $t$, proving invariance.

Finally, \Cref{eq:pg-reward-limit} implies, uniformly over strategies,
\begin{equation}
\sum_{i:\,z\in\mathcal M(x_i)}
c(x',x_i)R_\infty^\beta(x_i,z)
=\sum_{i=1}^n c(x',x_i)\mathbbm{1}_{z\in\mathcal M(x_i)}+o(\beta).
\end{equation}
Substituting this expression into \Cref{eq:pg-target-limit} shows that all mass
outside $\mathcal S(x')$ vanishes and that base-policy ratios within
$\mathcal S(x')$ are preserved. Thus
\begin{equation}
\pi_{\rm RL}(z|x')
=\frac{\pi_0(z|x')\mathbbm{1}_{z\in\mathcal S(x')}}
{\pi_0(\mathcal S(x')|x')}.
\label{eq:pg-target-low-beta}
\end{equation}
Summing over $\mathcal M(x')$ proves \Cref{eq:pg-target-correct-mass}.
\end{proof}

\begin{proof}[Proof of \Cref{cor:pg-target-reward}]
Every answer-parameter increment lies in the span of
$\phi(x_i,z)$ for $i\in[n]$ and $z\in\mathcal M(x_i)$.
\Cref{ass:answer-orthogonality} therefore fixes every target answer
logit. Combining this fact with \Cref{eq:pg-target-rate} shows that, for fixed
$\beta$, the reward converges at the same rate, and substitution of
\Cref{eq:pg-target-low-beta} gives \Cref{eq:pg-target-reward}.
\end{proof}

\subsection{Proof of SFT transferability}
\label{app:sft-forgetting}
\begin{proof}[Proof of \Cref{thm:sft-transfer}]
The strategy and answer gradients are bounded, so the flow exists for all time.
The case $M=1$ is immediate.  For $M\geq2$, write
$a_{i,z}(t)=\phi(x_i)^\top\theta_{t,z}^{\rm SFT}$. Orthogonality of the
training-prompt features makes the strategy dynamics independent across prompts:
\begin{equation}
\dot a_{i,z}
=\frac{\|\phi(x_i)\|^2}{n}
\left[\mathbbm{1}_{z=z_{E,i}}
-\pi_{\theta_t^{\rm SFT}}(z|x_i)\right].
\label{eq:sft-source-logit-flow}
\end{equation}
For every $k\ne z_{E,i}$,
\begin{equation}
\frac{d}{dt}e^{-a_{i,k}}
=\frac{\|\phi(x_i)\|^2/n}{\sum_z e^{a_{i,z}}}.
\end{equation}
Thus $e^{-a_{i,k}(t)}=e^{-a_{i,k}(0)}+u_i(t)$ for a common increasing function
$u_i(t)$ across all strategies other than $z_{E,i}$. The sum $\sum_z a_{i,z}(t)$ is
constant, and $u_i(t)\to\infty$; otherwise all logits would stay bounded
and $\dot u_i$ would stay bounded away from zero.  Substituting the
preceding expression into the conserved sum gives
\begin{equation}
u_i(t)^M\sim
\frac{M\|\phi(x_i)\|^2t}
{n\exp\!\left(\sum_z a_{i,z}(0)\right)}.
\end{equation}
Hence, for every $z$,
\begin{equation}
a_{i,z}(t)-a_{i,z}(0)
=\left(\mathbbm{1}_{z=z_{E,i}}-\frac1M\right)
\log\frac{M\|\phi(x_i)\|^2t}{n}
+\frac1M\sum_k a_{i,k}(0)-a_{i,z}(0)+o(1).
\label{eq:sft-source-logit-asymptotic}
\end{equation}

Each strategy-parameter increment lies in the span of the training-prompt
features, so the orthogonality of these features gives
\Cref{eq:projected-logit-change} with
$\theta_t$ replaced by $\theta_t^{\rm SFT}$.  Substituting
\Cref{eq:sft-source-logit-asymptotic} into this relation, exponentiating,
normalizing, and dropping factors independent of $z$ yields, uniformly over
the finite strategy set,
\begin{equation}
\pi_{\theta_t^{\rm SFT}}(z|x')
\propto \pi_0(z|x')
\prod_{i=1}^n\pi_0(z|x_i)^{-c(x',x_i)}
\prod_{i:\,z_{E,i}=z}
\left(\frac{M\|\phi(x_i)\|^2t}{n}\right)^{c(x',x_i)}
[1+o(1)].
\label{eq:sft-target-asymptotic}
\end{equation}
The power of $t$ is maximized exactly on $\mathcal S_E(x')$.  On this set,
the powers of $Mt/n$ are equal and cancel under normalization.  Removing
this common factor from \Cref{eq:sft-target-asymptotic} yields the explicit limit
\begin{equation}
\pi_{\rm SFT}(z|x')
=\frac{
\displaystyle
\pi_0(z|x')\mathbbm{1}_{z\in\mathcal S_E(x')}
\prod_{i=1}^n
\left[
\frac{\|\phi(x_i)\|^{2\mathbbm{1}_{z=z_{E,i}}}}
{\pi_0(z|x_i)}
\right]^{c(x',x_i)}
}{
\displaystyle
\sum_{k\in\mathcal S_E(x')}\pi_0(k|x')
\prod_{i=1}^n
\left[
\frac{\|\phi(x_i)\|^{2\mathbbm{1}_{k=z_{E,i}}}}
{\pi_0(k|x_i)}
\right]^{c(x',x_i)}
}.
\label{eq:sft-limit-policy}
\end{equation}
All remaining factors are positive, so the limiting policy is supported
exactly on $\mathcal S_E(x')$, proving \Cref{eq:sft-target-support}.
The two claims about probability mass on strategies valid for the target prompt follow from
this support characterization.
\end{proof}

%% file: arxiv/appendix/passn_boundary_proof.tex
\section{Proof for the Finite-Budget Coverage Theorem}
\label{app:passn}

\begin{proof}[Proof of \Cref{thm:passn-regimes}]
With the ideal answer decoder used to evaluate strategy success,
$R_\pi(x')=\pi(\mathcal M(x')|x')$, and
\Cref{thm:pg-transfer} gives
\begin{equation*}
R_{\pi_0}(x')=\pi_0(\mathcal M(x')|x'),
\qquad
R_{\pi_{\rm RL}}(x')
=\frac{\pi_0(\mathcal S(x')\cap\mathcal M(x')|x')}
{\pi_0(\mathcal S(x')|x')}.
\end{equation*}
Because the base softmax has full support, $R_{\pi_{\rm RL}}(x')=0$ exactly when
$\mathcal S(x')\cap\mathcal M(x')=\emptyset$.

For any random variable $U\in[0,1]$ and $0\leq a<b<1$,
\begin{equation*}
\Pr[U\leq a](1-a)^N
\leq\mathbb E[(1-U)^N]
\leq\Pr[U\leq b]+(1-b)^N.
\end{equation*}
The lower bound follows by restricting the expectation to $\{U\leq a\}$;
the upper bound follows by splitting it over $\{U\leq b\}$ and its complement.
The theorem's choice of $N$ ensures
$0<1/(4N)<\log(4/\alpha)/N<1$, and
\begin{equation*}
\left(1-\frac{1}{4N}\right)^N\geq\frac34,
\qquad
\left(1-\frac{\log(4/\alpha)}{N}\right)^N
<\frac{\alpha}{4}.
\end{equation*}

For part~\textnormal{(I)}, apply the lower bound to $R_{\pi_{\rm RL}}(x')$ with
$a=0$ and the upper bound to $R_{\pi_0}(x')$ with
$b=\log(4/\alpha)/N$.  This gives
\begin{equation*}
1-\operatorname{Pass@}N(\pi_{\rm RL})
\geq\alpha,
\qquad
1-\operatorname{Pass@}N(\pi_0)
<\frac{\alpha}{2}.
\end{equation*}
The difference between these failure probabilities gives the stated advantage
for the base policy.

For part~\textnormal{(II)}, apply the lower bound to $R_{\pi_0}(x')$ with
$a=1/(4N)$ and the upper bound to $R_{\pi_{\rm RL}}(x')$ with
$b=\log(4/\alpha)/N$.  This gives
\begin{equation*}
1-\operatorname{Pass@}N(\pi_0)
\geq\frac{3\alpha}{4},
\qquad
1-\operatorname{Pass@}N(\pi_{\rm RL})
<\frac{\alpha}{2}.
\end{equation*}
Their difference gives the stated advantage for the RL policy. Finally, as
$\beta\downarrow0$, the finite-$\beta$ success probabilities converge to
$R_{\pi_{\rm RL}}(x')$. By bounded convergence, the corresponding Pass@$N$
values converge under $\mathcal B$, so each strict bound persists for
sufficiently small $\beta>0$.
\end{proof}

%% file: arxiv/appendix/scaling_laws_proofs.tex
\section{Proofs for the Scaling Laws}
\label{app:scaling-proofs}

Throughout this appendix, we take $n=1$, let $x$ be the training prompt,
set $m:=|\mathcal M(x)|<M$, fix $\beta>0$ within each training run, and write
$C=t-T_0\geq0$.  For a target prompt $x'$, write
$c(x',x)=\langle\phi(x'),\phi(x)\rangle/\|\phi(x)\|^2$.  Write
$\mathsf{logit}(u)=\log(u/(1-u))$, $\sigma(s)=(1+e^{-s})^{-1}$, and
abbreviate
\begin{equation*}
\pi_x(C):=\pi_{T_0+C}(\mathcal M(x)|x),\qquad
\pi_{x'}(C):=\pi_{T_0+C}(\mathcal M(x)|x').
\end{equation*}
Thus $\pi_x(0)$ is the probability assigned to $\mathcal M(x)$ at $T_0$,
whereas $\pi_0$ denotes the fixed KL reference. Under
\Cref{ass:source-realizability,ass:scaling-balance}, there is a common
$\rho_\beta$ such that
\begin{equation*}
R_{T_0+C}(x,z)=\rho_\beta,
\qquad C\geq0,\quad z\in\mathcal M(x).
\end{equation*}
By \Cref{eq:pg-reward-limit},
$1-\rho_\beta=O(e^{-1/\beta})$. We use $\pi_x(\infty)$ and
$\pi_{x'}(\infty)$ for their fixed-$\beta$ limits.
We normalize the reward for a target prompt only when
$R_\infty^\beta(x')>R_{T_0}(x')$.
Throughout the vanishing-KL analysis below, $p:=\pi_x(0)$ and
$0<p<a<b<1$ are fixed. The term $o(1)$ is uniform on $[a,b]$ as
$\beta\downarrow0$. The remainder in the target log-linear approximation
is uniform over $c(x',x)>0$ and $\pi_{x'}(0)\in(0,1)$.

For $0<a<b<1$, set
\begin{equation*}
H(u):=\frac{2u-1}{u(1-u)}+2\mathsf{logit}(u).
\end{equation*}
Here $H'(u)=[u(1-u)]^{-2}$, and
\begin{equation}
\Delta_{a,b}^{(p)}:=\log\frac{H(b)-H(p)}{H(a)-H(p)}.
\label{eq:scaling-finite-clock-span}
\end{equation}

\subsection{Scalar training-prompt dynamics}

\begin{lemma}[Scalar training dynamics and transfer]
\label{lem:scaling-scalar-flow}
Under \Cref{ass:source-realizability,ass:scaling-balance}, the training-strategy
probability satisfies
\begin{align}
\frac{d\pi_x}{dC}
&=\|\phi(x)\|^2
\left(\frac1m+\frac1{M-m}\right)
[\pi_x(C)(1-\pi_x(C))]^2
\nonumber\\
&\quad\times
\left\{
\rho_\beta
-\beta\left[
\mathsf{logit}(\pi_x(C))
-\mathsf{logit}\left(\pi_0(\mathcal M(x)|x)\right)
\right]
\right\}.
\label{eq:scaling-scalar-flow}
\end{align}
Moreover,
\begin{equation}
\mathsf{logit}(\pi_x(\infty))
=\mathsf{logit}\left(\pi_0(\mathcal M(x)|x)\right)
+\frac{\rho_\beta}{\beta}.
\label{eq:scaling-source-endpoint}
\end{equation}
For every target prompt and $C\geq0$,
\begin{align}
\mathsf{logit}(\pi_{x'}(C))-\mathsf{logit}(\pi_{x'}(0))
&=c(x',x)
[\mathsf{logit}(\pi_x(C))-\mathsf{logit}(\pi_x(0))].
\label{eq:scaling-proof-exact-transfer}
\end{align}
Consequently,
\begin{equation}
\mathsf{logit}(\pi_{x'}(\infty))
=\mathsf{logit}(\pi_{x'}(0))
+c(x',x)\left[
\mathsf{logit}(\pi_x(\infty))-\mathsf{logit}(\pi_x(0))
\right].
\label{eq:scaling-target-endpoint}
\end{equation}
If \Cref{ass:answer-orthogonality} also holds, $c(x',x)>0$, and
$R_\infty^\beta(x')>R_{T_0}(x')$, then
\begin{equation}
\frac{R_{T_0+C}(x')-R_{T_0}(x')}
{R_\infty^\beta(x')-R_{T_0}(x')}
=\frac{\pi_{x'}(C)-\pi_{x'}(0)}
{\pi_{x'}(\infty)-\pi_{x'}(0)}.
\label{eq:scaling-proof-normalized-reward}
\end{equation}
\end{lemma}

\begin{proof}
For $\pi_z(C)=\pi_{T_0+C}(z|x)$, define
\begin{equation*}
U_z(C)=
\mathbbm{1}_{z\in\mathcal M(x)}\rho_\beta
-\beta\log\frac{\pi_z(C)}{\pi_0(z|x)}.
\end{equation*}
Differentiating the strategy objective gives
\begin{equation}
\frac{d}{dC}
\left[\theta_{T_0+C,z}^{\top}\phi(x)\right]
=\|\phi(x)\|^2\pi_z(C)
\left[U_z(C)-\sum_{k=1}^M\pi_k(C)U_k(C)\right].
\label{eq:scaling-logit-flow}
\end{equation}
At $C=0$, the KL reference and the policy at $T_0$ are each uniform within
the valid and invalid classes. By \Cref{ass:scaling-balance}, valid-strategy rewards
are also symmetric, while invalid strategies have zero return. Uniqueness of the
ODE solution therefore preserves uniformity within each class. Substituting
\begin{equation*}
\pi_z(C)=\frac{\pi_x(C)}m
\quad(z\in\mathcal M(x)),
\qquad
\pi_z(C)=\frac{1-\pi_x(C)}{M-m}
\quad(z\notin\mathcal M(x)),
\end{equation*}
into \Cref{eq:scaling-logit-flow} and differentiating the softmax yields
\Cref{eq:scaling-scalar-flow}.

Using the common conditional reward $\rho_\beta$ in
\Cref{eq:pg-mode-limit} gives
\begin{equation*}
\pi_\infty^\beta(z|x)\propto
\pi_0(z|x)
\exp\left(
\frac{\mathbbm{1}_{z\in\mathcal M(x)}\rho_\beta}{\beta}
\right).
\end{equation*}
Taking the odds ratio between the valid and invalid classes proves
\Cref{eq:scaling-source-endpoint}.

Because $n=1$, every strategy-parameter update is parallel to $\phi(x)$.
The resulting change in a strategy logit at $x'$ is therefore $c(x',x)$ times
its change at $x$. Taking the odds ratio between the valid and invalid classes proves
\Cref{eq:scaling-proof-exact-transfer}.  Letting $C\to\infty$ and applying
\Cref{eq:scaling-source-endpoint} gives \Cref{eq:scaling-target-endpoint}.

Finally, \Cref{ass:answer-orthogonality} fixes every target answer probability.
Within $\mathcal M(x)$ and its complement, the target strategy proportions
remain unchanged. Thus $R_{T_0+C}(x')$ depends linearly on $\pi_{x'}(C)$, up
to an additive constant. Subtracting its value at $C=0$ and dividing by its
total change at $C=\infty$ proves
\Cref{eq:scaling-proof-normalized-reward}.
\end{proof}

\begin{lemma}[Relating compute and reward in the strategy-dominated regime]
\label{lem:scaling-clock}
Fix $0<p<a<b<1$ and consider a family of runs satisfying
\Cref{ass:source-realizability,ass:scaling-balance}, with the reference
policy and $p,a,b$ fixed as $\beta\downarrow0$.
For sufficiently small $\beta$,
$R_\infty^\beta(x)>R_{T_0}(x)$ and
$\pi_x(C)$ increases through $[a,b]$.  For $u\in[a,b]$, let $C_u$ be the
first value of $C$ for which $\pi_x(C)=u$.  Uniformly for $u\in[a,b]$,
\begin{align}
\frac{\log(C_u/C_a)}{\log(C_b/C_a)}
&=\frac{\log\{[H(u)-H(p)]/[H(a)-H(p)]\}}{\Delta_{a,b}^{(p)}}+o(1),
\label{eq:scaling-clock-limit}\\
\frac{R_{T_0+C_u}(x)-R_{T_0}(x)}
{R_\infty^\beta(x)-R_{T_0}(x)}
&=\frac{u-p}{1-p}+o(1).
\label{eq:scaling-source-coordinate}
\end{align}
If, in addition, a target prompt $x'$ satisfies
\Cref{ass:answer-orthogonality}, with $c=c(x',x)>0$ and
$R_\infty^\beta(x')>R_{T_0}(x')$, then the following identity is exact:
\begin{equation}
\frac{R_{T_0+C_u}(x')-R_{T_0}(x')}
{R_\infty^\beta(x')-R_{T_0}(x')}
=\frac{\sigma\!\left(\mathsf{logit}(\pi_{x'}(0))
+c[\mathsf{logit}(u)-\mathsf{logit}(p)]\right)-\pi_{x'}(0)}
{\pi_{x'}(\infty)-\pi_{x'}(0)}.
\label{eq:scaling-target-coordinate}
\end{equation}
For fixed $c$ and $\pi_{x'}(0)\in(0,1)$, the denominator on the right-hand
side equals
$1-\pi_{x'}(0)+o(1)$.
\end{lemma}

\begin{proof}
Since $\rho_\beta\to1$ and the reference strategy mass is fixed, uniformly for
$p\leq v\leq b$,
\begin{equation*}
\rho_\beta-\beta\left[
\mathsf{logit}(v)-
\mathsf{logit}\left(\pi_0(\mathcal M(x)|x)\right)
\right]
=\rho_\beta[1+o(1)].
\end{equation*}
Hence the expression in braces in \Cref{eq:scaling-scalar-flow} is positive
on $[p,b]$ for sufficiently small $\beta$. On this interval,
$d\pi_x/dC$ is bounded below by a
positive multiple of $[p(1-b)]^2$, so $\pi_x(C)$ reaches $b$ in finite time.
Separating variables, integrating, and using
$H'(v)=[v(1-v)]^{-2}$ yields, uniformly for $u\in[a,b]$,
\begin{equation}
\|\phi(x)\|^2\left(\frac1m+\frac1{M-m}\right)
\rho_\beta C_u=[H(u)-H(p)][1+o(1)].
\label{eq:scaling-clock-comparison}
\end{equation}
In particular, $\log(C_b/C_a)=\Delta_{a,b}^{(p)}+o(1)$.
Since $\Delta_{a,b}^{(p)}>0$, taking logarithms and dividing by this span
proves \Cref{eq:scaling-clock-limit}. By
\Cref{eq:scaling-source-endpoint}, $\pi_x(\infty)\to1$.

Since $R_{T_0+C}(x)=\rho_\beta\pi_x(C)$ exactly,
\begin{equation*}
\frac{R_{T_0+C_u}(x)-R_{T_0}(x)}
{R_\infty^\beta(x)-R_{T_0}(x)}
=\frac{u-p}{\pi_x(\infty)-p}
=\frac{u-p}{1-p}+o(1).
\end{equation*}
This also shows $R_\infty^\beta(x)>R_{T_0}(x)$ and proves
\Cref{eq:scaling-source-coordinate}.

Combining the transfer identity
\Cref{eq:scaling-proof-exact-transfer} with the normalized reward identity
\Cref{eq:scaling-proof-normalized-reward} yields
\Cref{eq:scaling-target-coordinate}.  Finally,
\Cref{eq:scaling-target-endpoint} and fixed $c,\pi_{x'}(0)$ imply
$\pi_{x'}(\infty)\to1$.
\end{proof}

\subsection{Log-linear approximation}

\begin{lemma}[Linear approximation of an inverse map]
\label{lem:linear-log-secant}
Let $g\in C^2([a,b])$ satisfy $g'>0$ and
$|g''/g'|\leq K$.  There exists $\alpha\in\mathbb R$ such that
\begin{equation}
\sup_{u\in[a,b]}
\left|
u-\alpha-\frac{b-a}{g(b)-g(a)}[g(u)-g(a)]
\right|
\leq\frac{K(b-a)^2}{16}.
\label{eq:linear-log-secant}
\end{equation}
\end{lemma}

\begin{proof}
Normalize the interval and the function by
\begin{equation*}
F(t)=\frac{g(a+(b-a)t)-g(a)}{g(b)-g(a)},
\qquad f(t)=F'(t),\qquad d=K(b-a).
\end{equation*}
Then $f>0$, $\int_0^1f(t)\,dt=1$, and $\log f$ is $d$-Lipschitz.
If $d=0$, then $f\equiv1$ and the claim is exact.  Hence assume $d>0$.
Let $f^*$ be the decreasing rearrangement of $f$.  For any two values
$\lambda_1<\lambda_2$ in the range of $\log f$, continuity and the
$d$-Lipschitz property give
\begin{equation*}
\left|\{t\in[0,1]:\lambda_1<\log f(t)<\lambda_2\}\right|
\geq\frac{\lambda_2-\lambda_1}{d}.
\end{equation*}
Here $|\cdot|$ denotes Lebesgue measure. By continuity and the Lipschitz
condition, $\log f$ takes intermediate values over a set of at least the
displayed measure.
Decreasing rearrangement preserves these level-set measures, so $\log f^*$ is also
$d$-Lipschitz.  Hence
$f^*(t)/e^{-dt}$ is nondecreasing.  Integrating this ratio on
$[0,s]$ and $[s,1]$ gives
\begin{equation*}
\int_0^sf^*(t)\,dt
\leq\frac{1-e^{-ds}}{1-e^{-d}}.
\end{equation*}
Therefore
\begin{align}
\int_0^1(f(t)-1)_+\,dt
&=\max_{0\leq s\leq1}
\left\{\int_0^sf^*(t)\,dt-s\right\}
\nonumber\\
&\leq
\max_{0\leq s\leq1}
\left\{\frac{1-e^{-ds}}{1-e^{-d}}-s\right\}
\leq\frac d8.
\label{eq:linear-log-variation}
\end{align}
To verify the last inequality, a direct calculation shows that $d$ times the
preceding maximum equals
\begin{equation*}
\max_{0\leq v\leq1}
\{\log(1-v+ve^d)-vd\}.
\end{equation*}
For fixed $v$, this expression and its first derivative with respect to $d$
vanish at $d=0$, while its second derivative with respect to $d$ is at most
$1/4$.  Integrating this bound twice yields $d^2/8$.

Since $F(t)-t=\int_0^t(f(s)-1)\,ds$, its oscillation is at most the
right-hand side of
\Cref{eq:linear-log-variation}.  Returning to $[a,b]$, the residual in
\Cref{eq:linear-log-secant} therefore has oscillation at most
$K(b-a)^2/8$.  Centering its range proves the result.
\end{proof}

\begin{lemma}[Log-linear law for the training prompt]
\label{lem:source-linear-log}
Under \Cref{ass:source-realizability,ass:scaling-balance}, for any
$\varepsilon_{\rm fit}>0$ and sufficiently small $\beta$, there exist
$R_{0,x}<R_\infty^\beta(x)$ and $B_{{\rm lin},x},C_{{\rm ref},x}>0$
such that, for every $C\in[C_a,C_b]$,
\begin{equation}
\frac{\left|R_{T_0+C}(x)-R_{0,x}
-B_{{\rm lin},x}\log(C/C_{{\rm ref},x})\right|}
{R_\infty^\beta(x)-R_{0,x}}
\leq
\frac{b-a}{1-p}
\left[
\frac{b-a}{8}\max_{v\in\{a,b\}}\frac{|1-2v|}{v(1-v)}
+\frac{\Delta_{a,b}^{(p)}}{16}
\right]
+\varepsilon_{\rm fit}.
\end{equation}
\end{lemma}

\begin{proof}
Choose the fitted intercept $R_{0,x}:=R_{T_0}(x)$.
Since
\begin{equation*}
H'(u)=\frac{1}{[u(1-u)]^2},
\qquad
\left|\frac{H''(u)}{H'(u)}\right|
=\frac{2|1-2u|}{u(1-u)},
\end{equation*}
\Cref{lem:linear-log-secant} gives $\alpha\in\mathbb R$ such that
\begin{equation}
\sup_{u\in[a,b]}
\left|
u-\alpha-\frac{b-a}{H(b)-H(a)}[H(u)-H(a)]
\right|
\leq
\frac{(b-a)^2}{8}
\max_{u\in\{a,b\}}
\frac{|1-2u|}{u(1-u)}.
\label{eq:source-linear-clock-bound}
\end{equation}
Let $F(u)=\log(C_u/C_a)/\log(C_b/C_a)$ and
$G(u)=\log\{[H(u)-H(p)]/[H(a)-H(p)]\}/\Delta_{a,b}^{(p)}$.
By \Cref{eq:scaling-clock-limit}, $F=G+o(1)$ uniformly.  The identity
\begin{equation*}
\frac{H(u)-H(a)}{H(b)-H(a)}
=\frac{e^{\Delta_{a,b}^{(p)}G(u)}-1}{e^{\Delta_{a,b}^{(p)}}-1}
\end{equation*}
and \Cref{lem:linear-log-secant}, applied to $e^{\Delta_{a,b}^{(p)}s}$ on
$[0,1]$, show that, after adding a suitable constant, $F(u)$ approximates
this normalized change in $H$ with error at most
$\Delta_{a,b}^{(p)}/16+o(1)$.
Combining this bound with \Cref{eq:source-linear-clock-bound}, adjusting the
constant term, and using \Cref{eq:scaling-source-coordinate} gives
\begin{equation*}
\inf_{\alpha'\in\mathbb R}\sup_{u\in[a,b]}
\left|\frac{R_{T_0+C_u}(x)-R_{0,x}}
{R_\infty^\beta(x)-R_{0,x}}-\alpha'
-\frac{b-a}{1-p}F(u)\right|
\leq\frac{b-a}{1-p}\left[
\frac{b-a}{8}\max_{v\in\{a,b\}}
\frac{|1-2v|}{v(1-v)}+\frac{\Delta_{a,b}^{(p)}}{16}\right]+o(1).
\end{equation*}
Choose
\begin{equation*}
B_{{\rm lin},x}
=\frac{[R_\infty^\beta(x)-R_{0,x}](b-a)}
{(1-p)\log(C_b/C_a)}>0,
\end{equation*}
and choose $C_{{\rm ref},x}>0$ so that
$-B_{{\rm lin},x}\log C_{{\rm ref},x}$ equals the resulting constant term.
Because $\pi_x(C)$ is strictly increasing, $C_u$ covers $[C_a,C_b]$ as $u$
ranges over $[a,b]$. For sufficiently small $\beta$, the
normalized remainder is at most $\varepsilon_{\rm fit}$.
Multiplying the preceding bound by
$R_\infty^\beta(x)-R_{0,x}$ gives
\begin{equation*}
\begin{aligned}
&\sup_{C\in[C_a,C_b]}
\left|R_{T_0+C}(x)-R_{0,x}
-B_{{\rm lin},x}\log\frac{C}{C_{{\rm ref},x}}\right|\\
&\quad\leq [R_\infty^\beta(x)-R_{0,x}]\left\{
\frac{b-a}{1-p}\left[
\frac{b-a}{8}\max_{v\in\{a,b\}}\frac{|1-2v|}{v(1-v)}
+\frac{\Delta_{a,b}^{(p)}}{16}\right]+\varepsilon_{\rm fit}\right\}.
\end{aligned}
\end{equation*}
This proves the lemma.
\end{proof}

\begin{lemma}[Log-linear law for a target prompt]
\label{lem:target-linear-log}
Under \Cref{ass:source-realizability,ass:scaling-balance,ass:answer-orthogonality},
let $x'$ satisfy $c=c(x',x)>0$ and
$R_{T_0+C_b}(x')>R_{T_0+C_a}(x')$. For any
$\varepsilon_{\rm fit}>0$ and sufficiently small $\beta$, there exist
$B_{{\rm lin},x'},C_{{\rm ref},x'}>0$ such that
\begin{equation}
\sup_{C\in[C_a,C_b]}
\frac{\left|R_{T_0+C}(x')-R_{T_0}(x')
-B_{{\rm lin},x'}\log(C/C_{{\rm ref},x'})\right|}
{R_{T_0+C_b}(x')-R_{T_0+C_a}(x')}
\leq\frac{(1+c)\log\frac{b(1-a)}{a(1-b)}
+\Delta_{a,b}^{(p)}}{16}+\varepsilon_{\rm fit}.
\end{equation}
\end{lemma}

\begin{proof}
Write $Q(u)=\pi_{x'}(C_u)$.
The exact transfer identity gives
$\mathsf{logit}(Q(u))=\mathsf{logit}(Q(a))
+c[\mathsf{logit}(u)-\mathsf{logit}(a)]$.
As a function of $s=\mathsf{logit}(u)$,
\begin{equation*}
\frac{dQ}{ds}=cQ(1-Q),\quad
\left|\frac{d}{ds}\log\frac{dQ}{ds}\right|\leq c,\quad
\frac{dH}{ds}=\frac1{u(1-u)},\quad
\left|\frac{d}{ds}\log\frac{dH}{ds}\right|\leq1.
\end{equation*}
Applying \Cref{lem:linear-log-secant} first to $Q(\sigma(s))$ and then to
$H(\sigma(s))$, and combining the two approximations, gives a constant
$\alpha^\star$
such that
\begin{equation*}
\begin{aligned}
\sup_{u\in[a,b]}
\left|Q(u)-\alpha^\star
-(Q(b)-Q(a))\frac{H(u)-H(a)}{H(b)-H(a)}\right|
\leq\frac{[Q(b)-Q(a)](1+c)}{16}
\log\frac{b(1-a)}{a(1-b)}.
\end{aligned}
\end{equation*}
The proof of \Cref{lem:source-linear-log} shows that replacing the normalized
change in $H$ by $F(u)=\log(C_u/C_a)/\log(C_b/C_a)$ adds
at most $(Q(b)-Q(a))[\Delta_{a,b}^{(p)}/16+o(1)]$.
By \Cref{eq:scaling-target-coordinate}, multiplying by
$[R_{T_0+C_b}(x')-R_{T_0+C_a}(x')]/[Q(b)-Q(a)]$ converts the result
exactly into a bound for the unnormalized target reward, with no additional
error term.  Choose
\begin{equation*}
B_{{\rm lin},x'}
=\frac{R_{T_0+C_b}(x')-R_{T_0+C_a}(x')}{\log(C_b/C_a)}>0
\end{equation*}
and choose $C_{{\rm ref},x'}>0$ so that
$-B_{{\rm lin},x'}\log C_{{\rm ref},x'}$ equals the resulting constant term.
The error from replacing the normalized change in $H$ by $F(u)$ is
independent of $c$ and $\pi_{x'}(0)$. Hence, for sufficiently small $\beta$, this
error is at most $\varepsilon_{\rm fit}$ for every
target.  Thus
\begin{equation*}
\begin{aligned}
&\sup_{C\in[C_a,C_b]}
\left|R_{T_0+C}(x')-R_{T_0}(x')-B_{{\rm lin},x'}
\log\frac{C}{C_{{\rm ref},x'}}\right|\\
&\quad\leq [R_{T_0+C_b}(x')-R_{T_0+C_a}(x')]\left\{
\frac{1+c}{16}\log\frac{b(1-a)}{a(1-b)}
+\frac{\Delta_{a,b}^{(p)}}{16}+\varepsilon_{\rm fit}\right\}.
\end{aligned}
\end{equation*}
Dividing by the positive reward gain
$R_{T_0+C_b}(x')-R_{T_0+C_a}(x')$ proves the lemma.
\end{proof}

\subsection{Benchmark log-sigmoid approximation}
\label{app:benchmark-scaling}

We prove the benchmark log-sigmoid result of \Cref{sec:scaling}. Let
$\mathcal B$ be a benchmark distribution over target prompts, and write
$R(t):=\E_{x'\sim\mathcal B}[R_t(x')]$.
The fitted lower level $R_0$ below need not equal the base-model score
$R(0)$ or the observed score $R(T_0)$ before fitting begins.
The fitted upper level $A$ need not equal the fixed-$\beta$ limiting reward.

For $x'\sim\mathcal B$, let
\begin{align*}
r_+(x')&:=\mathbb E_{\pi_{T_0}(\cdot|x')}
  [R_{T_0}(x',z)\mid z\in\mathcal M(x)],\\
r_-(x')&:=\mathbb E_{\pi_{T_0}(\cdot|x')}
  [R_{T_0}(x',z)\mid z\notin\mathcal M(x)],\\
d(x')&:=r_+(x')-r_-(x').
\end{align*}
The conditional expectations are well defined because the softmax policy has
full support and $0<|\mathcal M(x)|<M$.
When $c(x',x)>0$, the mass on $\mathcal M(x)$ increases across the fitting
interval, so the positive reward gain assumed in the main text is equivalent
to $d(x')>0$.
For $0<p<a<b<1$, set
\begin{equation}
\delta_{a,b}:=\mathsf{logit}(b)-\mathsf{logit}(a)
=\log\frac{b(1-a)}{a(1-b)},
\qquad
\Gamma_{a,b}^{(p)}:=\frac12\inf_{\lambda>0}
\operatorname{osc}_{u\in[a,b]}
\left\{\mathsf{logit}(u)-\lambda\log[H(u)-H(p)]\right\},
\label{eq:benchmark-scaling-source-error}
\end{equation}
where $p=\pi_{T_0}(\mathcal M(x)|x)$ and $H$ is defined above. Thus
$\delta_{a,b}$ is the training-strategy log-odds span,
and $\Gamma_{a,b}^{(p)}$ measures the best approximation of those log odds by a
linear function of log compute.
Set
$R_0:=\E_{x'\sim\mathcal B}[r_-(x')]$,
$D:=\E_{x'\sim\mathcal B}[d(x')]$,
and $A:=R_0+D$.

\begin{thm}[Benchmark log-sigmoid approximation under bounded feature overlap]
\label{thm:benchmark-bounded-overlap-scaling}
Assume \Cref{ass:source-realizability,ass:scaling-balance} with one training
prompt and benchmark distribution $\mathcal B$. Suppose
\Cref{ass:answer-orthogonality} holds for $x'\sim\mathcal B$.
Consider the vanishing-$\beta$ family of \Cref{lem:scaling-clock}.
Suppose $c(x',x),d(x')>0$ for $x'\sim\mathcal B$, and let $c_{\max}>0$
be a finite upper bound on $c(x',x)$ for these prompts. For any
$0<c_0\leq c_{\max}$ and $\varepsilon_{\rm fit}>0$, sufficiently small
$\beta$ admits $B,C_{\rm mid}>0$ such that
\begin{align}
\sup_{C\in[C_a,C_b]}
\left|\frac{R(T_0+C)-R_0}{A-R_0}
-\frac{1}{1+(C_{\rm mid}/C)^B}\right|
&\leq\tanh\left(\frac{c_0^2\delta_{a,b}^2}{256}
+\frac{c_0\Gamma_{a,b}^{(p)}}4\right)\nonumber\\
&\quad+\frac{\E_{x'\sim\mathcal B}\left[
d(x')\tanh\left(\delta_{a,b}(c(x',x)-c_0)_+/8\right)\right]}
{\E_{x'\sim\mathcal B}[d(x')]}+\varepsilon_{\rm fit},
\label{eq:benchmark-bounded-overlap-scaling}
\end{align}
where $(q)_+:=\max\{q,0\}$.
For $c_0=c_{\max}$ and fixed $p<a<1$, the corresponding raw-reward
error is $O((b-a)^2)$ as $b-a\to0$, provided $\beta$ is sufficiently small.
\end{thm}

\begin{proof}[Proof of \Cref{thm:benchmark-bounded-overlap-scaling}]
Fix $0<c_0\leq c_{\max}$.
Let $z(C)=\mathsf{logit}(\pi_{T_0+C}(\mathcal M(x)|x))-\mathsf{logit}(p)$.
By \Cref{lem:scaling-scalar-flow}, the target mass of $\mathcal M(x)$ is
\begin{equation*}
\pi_{T_0+C}(\mathcal M(x)|x')
=\sigma\left(\mathsf{logit}(\pi_{T_0}(\mathcal M(x)|x'))
+c(x',x)z(C)\right).
\end{equation*}
Target answer-feature orthogonality fixes each conditional reward at its
base value, independently of $t$ and $\beta$.
By \Cref{eq:pg-transfer-identity} and training-strategy symmetry, the target's
strategy proportions within $\mathcal M(x)$ and its complement equal their base
values at $T_0$ and remain fixed thereafter. Thus $r_\pm(x')$ and $d(x')$
do not vary with $\beta$. Under the theorem's assumptions, $D>0$, and
\begin{equation*}
R(T_0+C)=R_0+D F(z(C)),
\qquad
F(z):=\frac{\E_{x'\sim\mathcal B}\left[
d(x')\sigma\left(\mathsf{logit}(\pi_{T_0}(\mathcal M(x)|x'))
+c(x',x)z\right)\right]}{D}.
\end{equation*}
In particular, $A-R_0=D>0$.
Put $z_{\rm mid}:=[z(C_a)+z(C_b)]/2$ and define an auxiliary mixture by
replacing each logistic slope $c(x',x)$ with $\min\{c(x',x),c_0\}$:
\begin{equation*}
\widetilde F(z):=\frac{\E_{x'\sim\mathcal B}\left[
d(x')\sigma\left(
\mathsf{logit}(\pi_{T_0}(\mathcal M(x)|x'))
+(c(x',x)-c_0)_+z_{\rm mid}+\min\{c(x',x),c_0\}z
\right)\right]}{D}.
\end{equation*}
The offset $(c(x',x)-c_0)_+z_{\rm mid}$ makes the original and auxiliary
logits agree at $z=z_{\rm mid}$. Since
$|z-z_{\rm mid}|\leq\delta_{a,b}/2$ for $z\in[z(C_a),z(C_b)]$ and
$|\sigma(s)-\sigma(t)|\leq\tanh(|s-t|/4)$,
\begin{equation}
\sup_{z\in[z(C_a),z(C_b)]}|F(z)-\widetilde F(z)|
\leq
\frac{\E_{x'\sim\mathcal B}\left[
d(x')\tanh\left(\frac{\delta_{a,b}(c(x',x)-c_0)_+}{8}\right)
\right]}{D}.
\label{eq:benchmark-slope-replacement-error}
\end{equation}

We next bound the curvature of the auxiliary mixture. Write
$L(z):=\mathsf{logit}(\widetilde F(z))$. Any positive mixture of logistic
functions with slopes in $(0,c_0]$ satisfies
\begin{equation}
0<L'(z)\leq c_0,
\qquad |L''(z)|\leq\frac{c_0^2}{4}.
\label{eq:benchmark-mixture-curvature}
\end{equation}
The derivatives of the integrand up to order two are uniformly bounded on
compact $z$-intervals because $0<d(x')\leq1$ and
$0<c(x',x)\leq c_{\max}$. Dominated convergence therefore permits
differentiation under the benchmark expectation.
For completeness, let $t=e^{c_0z}$. Each component $q_{x'}(t)$ of
$\widetilde F$ has odds proportional to $t^{s_{x'}}$ for some
$s_{x'}\in(0,1]$, so $q_{x'}/(1-q_{x'})$ is concave in $t$. Hence
$(1-q_{x'})\partial_{tt}q_{x'}+2(\partial_tq_{x'})^2\leq0$.
Averaging with weights $d$ and applying the Cauchy--Schwarz inequality gives
\begin{equation*}
\left(\frac{\E_{x'\sim\mathcal B}[d(x')\partial_tq_{x'}]}{D}\right)^2
\leq (1-\widetilde F)
\frac{\E_{x'\sim\mathcal B}
[d(x')(\partial_tq_{x'})^2/(1-q_{x'})]}{D}.
\end{equation*}
Consequently,
\begin{equation*}
(1-\widetilde F)\partial_{tt}\widetilde F
+2(\partial_t\widetilde F)^2\leq0,
\end{equation*}
so $\widetilde F/(1-\widetilde F)$ is concave in $t$.
The same argument applied to the complementary components in
$t^{-1}$ shows that $(1-\widetilde F)/\widetilde F$ is concave in
$t^{-1}$. Also,
$\partial_z\widetilde F=\E_{x'\sim\mathcal B}
[d(x')s_{x'}c_0q_{x'}(1-q_{x'})]/D$, and
$\E_{x'\sim\mathcal B}[d(x')q_{x'}(1-q_{x'})]/D
\leq\widetilde F(1-\widetilde F)$;
thus $0<L'\leq c_0$.
Writing $h=L'/c_0\in(0,1]$, concavity of the odds in $t$ yields
$L''/c_0^2\leq h(1-h)$, while concavity of the complementary odds in
$t^{-1}$ yields $-L''/c_0^2\leq h(1-h)$. This proves
\Cref{eq:benchmark-mixture-curvature}.

Let $\ell$ interpolate $L$ at
$z_{\rm mid}-\delta_{a,b}/(2\sqrt2)$ and
$z_{\rm mid}+\delta_{a,b}/(2\sqrt2)$.
The interpolation remainder and \Cref{eq:benchmark-mixture-curvature}
give
\begin{equation*}
\sup_{z\in[z(C_a),z(C_b)]}|L(z)-\ell(z)|
\leq\frac{c_0^2\delta_{a,b}^2}{64},
\qquad 0<\ell'\leq c_0.
\end{equation*}
By the definition of $\Gamma_{a,b}^{(p)}$, for every $\eta>0$ there
exist $\lambda>0$ and $\gamma\in\mathbb R$ such that
\begin{equation*}
\sup_{u\in[a,b]}
\left|\mathsf{logit}(u)-\mathsf{logit}(p)
-\gamma-\lambda\log[H(u)-H(p)]\right|
\leq \Gamma_{a,b}^{(p)}+\eta.
\end{equation*}
The relation between compute and strategy probability in \Cref{lem:scaling-clock} gives
$\log(C_u/C_a)=\log\{[H(u)-H(p)]/[H(a)-H(p)]\}+o(1)$
uniformly on $[a,b]$. Consequently,
$\ell\bigl(\gamma+\lambda\log[H(a)-H(p)]+\lambda\log(C/C_a)\bigr)$
is affine in $\log C$ with positive slope and approximates $L(z(C))$ within
$c_0^2\delta_{a,b}^2/64+c_0\Gamma_{a,b}^{(p)}+c_0\eta+o(1)$.
Applying $\sigma$ gives a log-sigmoid fit to $\widetilde F(z(C))$
with error at most
$\tanh(c_0^2\delta_{a,b}^2/256+c_0\Gamma_{a,b}^{(p)}/4
+c_0\eta/4)+o(1)$.

The normalized benchmark reward in \Cref{eq:benchmark-bounded-overlap-scaling} is
exactly $F(z(C))$. Combine this with \Cref{eq:benchmark-slope-replacement-error} and let
$\eta\downarrow0$; the remaining uniform $o(1)$ is absorbed into
$\varepsilon_{\rm fit}$.
As the fitting interval narrows, $\delta_{a,b}=O(b-a)$. In the definition of
$\Gamma_{a,b}^{(p)}$, choose $\lambda$ so that
$\mathsf{logit}(u)-\lambda\log[H(u)-H(p)]$ has zero derivative at $a$;
Taylor's theorem then gives $\Gamma_{a,b}^{(p)}=O((b-a)^2)$. Taking
$c_0=c_{\max}$ and $\varepsilon_{\rm fit}=(b-a)^2$ makes the normalized
error $O((b-a)^2)$. Since $A-R_0\leq1$, so is the raw error.
\end{proof}

\begin{proof}[Proof of \Cref{thm:benchmark-sigmoid-scaling}]
Take $c_0=c_{\max}$ in \Cref{thm:benchmark-bounded-overlap-scaling}; the second
term vanishes. This gives the theorem directly when $R_0>0$. If $R_0=0$,
apply that bound with $\varepsilon_{\rm fit}/2$ and choose a positive fitted
lower level $\eta<A$ instead. Since $R(T_0+C)\leq A$ over the fitting interval,
\begin{equation*}
\left|\frac{R(T_0+C)-\eta}{A-\eta}-\frac{R(T_0+C)}{A}\right|
\leq\frac{\eta}{A-\eta},
\end{equation*}
so a sufficiently small $\eta$ absorbs the difference into the remaining
$\varepsilon_{\rm fit}/2$.
To obtain this rate as the fitting interval narrows, take $\eta=O((b-a)^2)$.
\end{proof}

\subsection{Benchmark log-linear approximation}
\label{app:benchmark-linear-scaling}

The single-target bound also applies to the benchmark average. Each target
prompt's feature overlap is weighted by its reward gain over the fitting interval:
\begin{equation}
\bar c:=
\frac{\E_{x'\sim\mathcal B}\left[
c(x',x)\bigl(R_{T_0+C_b}(x')-R_{T_0+C_a}(x')\bigr)\right]}
{R(T_0+C_b)-R(T_0+C_a)}.
\label{eq:benchmark-average-overlap}
\end{equation}
The log-linear theorem assumes $\bar c<\infty$; unlike the log-sigmoid
theorem, it does not require a uniform bound on feature overlap.
These reward gains depend on $d(x')$,
$\pi_{T_0}(\mathcal M(x)|x')$, and $c(x',x)$ through the exact transfer
identity above.

\begin{thm}[Benchmark log-linear law]
\label{thm:benchmark-linear-scaling-formal}
Assume \Cref{ass:source-realizability,ass:answer-orthogonality,ass:scaling-balance}
with one training prompt and $0<p<a<b<1$. Suppose
$c(x',x),d(x')>0$ for $x'\sim\mathcal B$ and $\bar c<\infty$. For any
$\varepsilon_{\rm fit}>0$ and sufficiently small $\beta$, there exist
$R_0,B_{\rm lin},C_{\rm ref}>0$ such that
\begin{equation}
\sup_{C\in[C_a,C_b]}
\frac{\left|R(T_0+C)-R_0-B_{\rm lin}\log(C/C_{\rm ref})\right|}
{R(T_0+C_b)-R(T_0+C_a)}
\leq\frac{(1+\bar c)\delta_{a,b}+\Delta_{a,b}^{(p)}}{16}
+\varepsilon_{\rm fit}.
\label{eq:benchmark-linear-scaling-formal}
\end{equation}
If $c(x',x)\leq c_{\max}<\infty$ on the benchmark, the raw-reward error
is $O((b-a)^2)$ for fixed $p<a<1$ as $b-a\to0$, provided $\beta$ is
sufficiently small.
\end{thm}

\begin{proof}[Proof of \Cref{thm:benchmark-linear-scaling-formal}]
Apply \Cref{lem:target-linear-log} for $x'\sim\mathcal B$.
The error in approximating the relation between compute and strategy probability is uniform in $c(x',x)$ and the initial
target strategy mass, so a single sufficiently small $\beta$ makes the promptwise bounds
hold for almost every prompt in the benchmark distribution. The fitted slope is
\begin{equation*}
B_{{\rm lin},x'}
=\frac{R_{T_0+C_b}(x')-R_{T_0+C_a}(x')}{\log(C_b/C_a)}>0.
\end{equation*}
For this slope, choose the fitted constant term as the midpoint of the
range of $R_{T_0+C}(x')-B_{{\rm lin},x'}\log C$ on $[C_a,C_b]$.
This choice minimizes the uniform promptwise error.
Because rewards lie in $[0,1]$ and $C_a,C_b$ are fixed, these slopes and
constant terms are bounded and may be averaged. Averaging yields a
log-linear fit with positive slope
$B_{\rm lin}=\E_{x'\sim\mathcal B}[B_{{\rm lin},x'}]$.
For any $R_0>0$, choose $C_{\rm ref}>0$ so that the averaged constant term
equals $R_0-B_{\rm lin}\log C_{\rm ref}$.
The triangle inequality bounds its error by the expectation of the
promptwise errors. Each is proportional to the positive gain
$R_{T_0+C_b}(x')-R_{T_0+C_a}(x')$. Dividing by the benchmark gain and using
\Cref{eq:benchmark-average-overlap} proves the displayed bound.
As the fitting interval narrows, $\bar c\leq c_{\max}$, while
$\delta_{a,b}$ and $\Delta_{a,b}^{(p)}$ are $O(b-a)$. The exact transfer
identity gives
\begin{equation*}
\frac{d}{du}\pi_{x'}(C_u)
=\frac{c(x',x)\pi_{x'}(C_u)(1-\pi_{x'}(C_u))}{u(1-u)},
\end{equation*}
so the benchmark reward gain is $O(b-a)$ because $c(x',x)\leq c_{\max}$.
Taking $\varepsilon_{\rm fit}=b-a$ proves the raw-reward claim.
\end{proof}

\subsection{Visualizing the approximation bounds}
\label{app:scaling-bound-magnitudes}

We plot the approximation error bounds in
\Cref{thm:benchmark-bounded-overlap-scaling,thm:benchmark-linear-scaling-formal}.
For a single target with unit reward contrast,
\Cref{fig:scaling-bound-heatmaps} shows the bounds for different fitting intervals.

\begin{figure}[t]
\centering
\includegraphics[width=0.86\linewidth]
{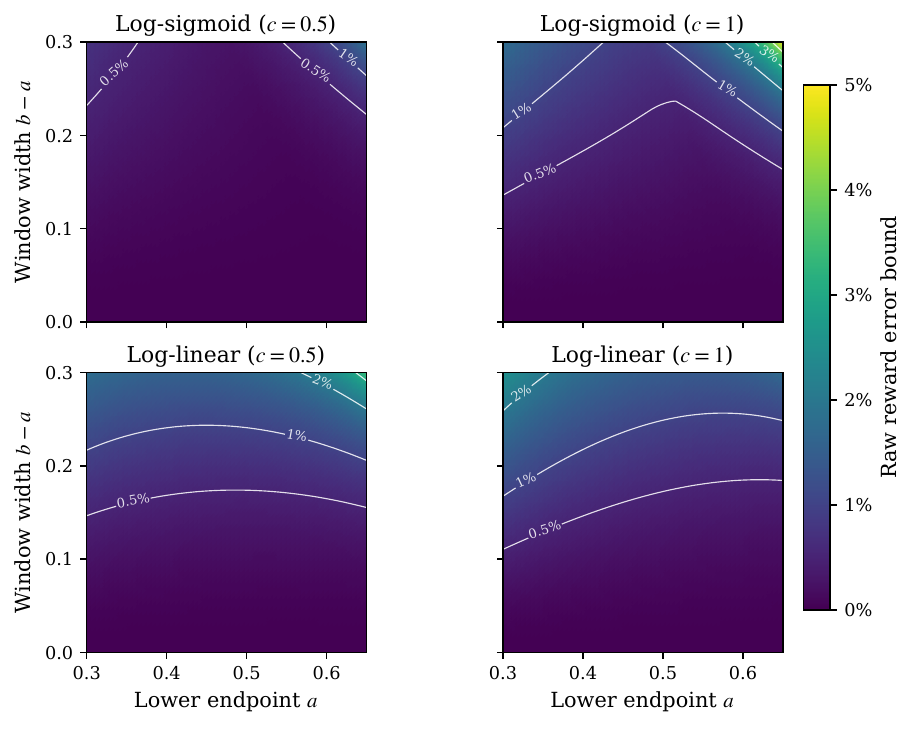}
\caption{\textbf{The theoretical approximation bounds are a few percentage
points in a representative regime.}
We use $p=0.1$, initial target-strategy mass $q_0=0.5$, unit reward contrast,
and feature overlap $c\in\{0.5,1\}$.  We vary $a\in[0.30,0.65]$ and
$b-a\in[0,0.30]$, with the zero-width boundary shown by continuity.  All
panels use the same raw-reward color scale, and contours report
percentage-point error.  The asymptotic term $\varepsilon_{\rm fit}$ is
omitted.}
\label{fig:scaling-bound-heatmaps}
\end{figure}

Neither functional form is uniformly favored. For $c=0.5$, the displayed log-sigmoid bound is smaller throughout the
plotted range. For $c=1$, it remains smaller for most of the plotted intervals, but the
log-linear bound becomes smaller when both $a$ and $b-a$ are large;
for example, at $(a,b-a)=(0.65,0.30)$ the two bounds are approximately
$4.90\%$ and $1.75\%$, respectively.

%% file: arxiv/appendix/experiment_details.tex
\section{Experiment Details and Further Results}
\label{app:experiment-details}

\subsection{Training setup}

\paragraph{RL.}
We run GRPO in verl with a learning rate of $10^{-6}$ for 200 steps.
Each step samples 512 prompts and 16 responses per prompt, followed by one PPO epoch
with a minibatch size of 128 prompts. Training rollouts use temperature 1,
top-$p$ 1, a 2,048-token prompt limit, and a 16,384-token response limit.
The lower and upper clipping parameters are 0.2 and 0.28. Mathematics rewards indicate
final-answer correctness, while coding rewards measure the fraction of
ten tests passed.

\paragraph{SFT.}
We use batch size 512, learning rate of $5\times10^{-5}$,
a two-epoch cosine schedule with 10\% warmup, and a 16,384-token sequence
limit. We evaluate the model after 1.5 epochs of training (step 117).

\subsection{Transfer and forgetting}
\label{app:evaluation-protocols}

\paragraph{Benchmarks.}
We evaluate transfer and forgetting across four domains:
\begin{enumerate}[nosep,leftmargin=*]
\item \textbf{Math}: MATH500~\citep{hendrycks2021math}, AIME 2024, AIME 2025, and OMEGA~\citep{sun2025omega};
\item \textbf{Coding}: LiveCodeBench~\citep{jain2024livecodebench}, CodeContests~\citep{li2022alphacode}, and DELTA~\citep{sun2025rlgrok};
\item \textbf{Knowledge and Planning}: GPQA-Diamond~\citep{rein2023gpqa}, HEAD-QA~\citep{vilares2019headqa}, ACPBench~\citep{kokel2025acpbench}, ToolUse~\citep{tang2023toolalpaca}, and Science~\citep{feng2024sciknoweval};
\item \textbf{General Capabilities}: CoQA~\citep{reddy2019coqa}, IFEval~\citep{zhou2023ifeval}, HaluEval~\citep{li2023halueval}, MC-TACO~\citep{zhou2019mctaco}, WildJailbreak~\citep{jiang2024wildteaming}, and WildGuardTest~\citep{han2024wildguard}.
\end{enumerate}

\paragraph{Evaluation protocol.}
We use native chat templates. Sampled reasoning tasks use temperature 0.6,
top-$p$ 0.95, top-$k$ 20, and a maximum generation length of 32,768 tokens.
We sample three responses per problem for MATH500 and GPQA-Diamond, 10 for
LiveCodeBench, 32 for AIME, and one for DELTA. CodeContests uses Eurus-2-RL's
377-problem validation subset. For greedy decoding, the maximum response
length is 32,768 tokens for ACPBench, CoQA, and HaluEval, and 4,096 tokens for
ToolUse, Science, and the safety benchmarks.
HEAD-QA and MC-TACO use likelihood scoring; IFEval uses one response.
We retain native task metrics, including strict instruction-level IFEval
accuracy, CoQA F1, DELTA test pass rates, and WildGuard-based refusal accuracy.

\paragraph{Complete scores.}
\Cref{tab:transfer-forgetting} reports the absolute scores underlying
\Cref{fig:transfer-forgetting}.

\input{arxiv/tables/transfer_forgetting}

\subsection{Pass@$N$}

\paragraph{Evaluation protocol.}
For each problem in MATH500, AIME 2024, AIME 2025, GPQA-Diamond, CodeContests,
and BoxNet, we sample 1,024 independent responses using temperature 0.7,
top-$p$ 0.95, and a maximum generation length of 16,384 tokens.

\paragraph{Complete results.}
\Cref{tab:passn-full} reports the coverage results at each sampling budget
underlying \Cref{fig:passn-qwen-gemma}.

\input{arxiv/tables/passn_full}

\subsection{Scaling laws}
\label{app:scaling-fit-procedure}

\paragraph{Fitting procedure.}
For each benchmark, let $(C_j,R_j)_{j=1}^m$ be its positive-compute
evaluations, where $C$ is cumulative compute in units of $10^2$ GPU-hours
and scores lie in $[0,1]$.  We fix $R_0$ to the measured base-model score.
The zero-compute evaluation therefore determines $R_0$ but is excluded from
both fitting losses because $\log C$ is undefined at $C=0$.  All
positive-compute evaluations receive equal weight.

For the log-sigmoid law in \Cref{eq:empirical-log-sigmoid}, we solve
\begin{equation}
\bigl(\widehat A,\widehat B,\widehat C_{\rm mid}\bigr)
\in
\argmin_{\substack{0\leq A\leq1,\;B>0,\\ C_{\rm mid}>0}}
\sum_{j=1}^m
\left[
R_j-R_0-
\frac{A-R_0}{1+(C_{\rm mid}/C_j)^B}
\right]^2.
\label{eq:empirical-log-sigmoid-fit}
\end{equation}
To solve this problem, set $q_j=[1+(C_{\rm mid}/C_j)^B]^{-1}$.  For fixed
$(B,C_{\rm mid})$, the unconstrained optimum is
$R_0+\sum_j q_j(R_j-R_0)/\sum_j q_j^2$; we clip it to $[0,1]$ to obtain
$A$.  We then minimize the profiled loss over
$(\log C_{\rm mid},\log B)$ using the Nelder--Mead method with bounds and multiple initializations. We search
$\log B\in[-14,7]$ and use
$C_{\rm mid}\in[10^{-12},10^{28}]$ in normalized compute units.  This
finite numerical range is effectively unbounded for the reported data, and
none of the fits reaches a midpoint boundary.  Among all local solutions,
we retain the one with the smallest squared error; the numerical search does
not certify global optimality.

For the log-linear law in \Cref{eq:empirical-log-linear}, we first fit an
unrestricted intercept and slope by ordinary least squares of $R_j$ on
$\log C_j$.  Because the slopes in all reported fits are nonzero, this is
equivalent to solving
\begin{equation}
\bigl(\widehat B_{\rm lin},\widehat C_{\rm ref}\bigr)
\in
\argmin_{\substack{B_{\rm lin}\in\mathbb R,\\ C_{\rm ref}>0}}
\sum_{j=1}^m
\left[
R_j-R_0-B_{\rm lin}\log(C_j/C_{\rm ref})
\right]^2.
\label{eq:empirical-log-linear-fit}
\end{equation}
Here $\widehat B_{\rm lin}$ is the fitted slope, and
$\widehat C_{\rm ref}$ is the unique compute level at which the fitted line equals
$R_0$.

The step-based fits reported below use the same
procedure with optimizer step as compute.  Steps are divided by 200 for
sigmoid conditioning, and the fitted compute scales are converted back to
steps.

For either fitted curve $\widehat R(C)$, we report
\begin{equation}
\mathsf{R^2}
=1-
\frac{\sum_{j=1}^m\bigl[R_j-\widehat R(C_j)\bigr]^2}
{\sum_{j=1}^m(R_j-\overline R)^2},
\qquad
\overline R=\frac1m\sum_{j=1}^m R_j,
\label{eq:empirical-scaling-r2}
\end{equation}
using the same positive-compute evaluations as in the corresponding fit;
this quantity is defined whenever the denominator is nonzero.

\paragraph{GPU-hour fits.}
\Cref{fig:qwen-scaling-five} shows the fitted scaling curves, and
\Cref{tab:qwen-scaling-five} reports their coefficients.

\input{arxiv/figures/qwen3_14b_selected_five_benchmarks_figure}

\input{arxiv/tables/qwen_scaling_five}

\paragraph{Scaling with training steps.}
As a robustness check, we use optimizer step as the compute metric. We fit
both laws through step 200 and evaluate their extrapolations through step 350.
\Cref{fig:qwen-scaling-five-steps} shows the fitted curves using optimizer
steps, and \Cref{tab:qwen-scaling-five-steps} reports their coefficients.
\Cref{fig:qwen-scaling-extrapolation-steps} shows the corresponding fitting
and extrapolation regimes. Although GPU-hours per step increase as average
rollout lengths grow,
$\log(\textrm{GPU Hours})$ is approximately linear in
$\log(\textrm{Steps})$ over these runs; the two compute metrics therefore
give similar qualitative conclusions.

\input{arxiv/figures/qwen3_14b_selected_five_benchmarks_log_steps_figure}

\input{arxiv/figures/qwen3_14b_pre200_step_extrapolation_figure}

\input{arxiv/tables/qwen_scaling_five_steps}

For Math RL on LiveCodeBench, the step-based fit reaches $A=1.000$ and gives
the unconstrained estimate $C_{\rm mid}=2.0\times10^4$, far beyond the
observed range; its midpoint should therefore be treated as weakly identified.

%% file: arxiv/tables/transfer_forgetting.tex
% Generated by figures/generate_transfer_forgetting.py
% Registry SHA-256: 4fb42dc6b84315b49d9ca68c5ff4f2eabc578d2ef3bc4f7b94df0d347633a281
\begin{table}[H]
\centering
\caption{\textbf{Mathematics gains accompany transfer to some tasks and forgetting on others.}
Scores (\%; higher is better) for the base models, SFT models evaluated with
thinking enabled, and RL models at step 200, rounded to one decimal.
Colors indicate changes from the corresponding base model before rounding.}
\label{tab:transfer-forgetting}
\begingroup
\definecolor{transfergain}{HTML}{087F8C}
\definecolor{transferloss}{HTML}{B44C36}
\definecolor{transferband}{HTML}{EEF2F5}
\definecolor{transferink}{HTML}{31445A}
\fontsize{9}{10.6}\selectfont
\setlength{\tabcolsep}{0pt}
\renewcommand{\arraystretch}{1.13}
\begin{tabular}{@{}p{0.211\linewidth}*{2}{>{\centering\arraybackslash}p{0.080\linewidth}}>{\centering\arraybackslash}p{0.120\linewidth}>{\centering\arraybackslash}p{0.165\linewidth}@{\hspace{0.026\linewidth}}>{\centering\arraybackslash}p{0.100\linewidth}>{\centering\arraybackslash}p{0.095\linewidth}>{\centering\arraybackslash}p{0.123\linewidth}@{}}
\toprule
& \multicolumn{4}{c}{\textbf{Qwen3-14B-Base}} & \multicolumn{3}{c}{\textbf{Gemma-3-12B-IT}} \\
\cmidrule(lr){2-5}\cmidrule(l){6-8}
\textbf{Benchmark} & \mbox{Base} & \mbox{SFT} & \mbox{Math RL} & \mbox{Math-Code RL} & \mbox{Base} & \mbox{SFT} & \mbox{Math RL} \\
\midrule
\rowcolor{transferband}\multicolumn{8}{@{}l}{\strut\textcolor{transferink}{\textbf{Math}}} \\
\addlinespace[2pt]
MATH500 & \makebox[2em][r]{75.5} & \makebox[2em][r]{\textcolor{transfergain}{87.7}} & \makebox[2em][r]{\textcolor{transfergain}{87.4}} & \makebox[2em][r]{\textcolor{transfergain}{88.3}} & \makebox[2em][r]{80.9} & \makebox[2em][r]{\textcolor{transferloss}{80.3}} & \makebox[2em][r]{\textcolor{transfergain}{82.9}} \\
AIME 2024 & \makebox[2em][r]{14.3} & \makebox[2em][r]{\textcolor{transfergain}{63.0}} & \makebox[2em][r]{\textcolor{transfergain}{33.1}} & \makebox[2em][r]{\textcolor{transfergain}{46.2}} & \makebox[2em][r]{25.9} & \makebox[2em][r]{\textcolor{transfergain}{31.4}} & \makebox[2em][r]{\textcolor{transfergain}{28.4}} \\
AIME 2025 & \makebox[2em][r]{14.0} & \makebox[2em][r]{\textcolor{transfergain}{48.7}} & \makebox[2em][r]{\textcolor{transfergain}{25.2}} & \makebox[2em][r]{\textcolor{transfergain}{31.4}} & \makebox[2em][r]{17.4} & \makebox[2em][r]{\textcolor{transfergain}{23.9}} & \makebox[2em][r]{\textcolor{transfergain}{21.8}} \\
OMEGA & \makebox[2em][r]{11.4} & \makebox[2em][r]{\textcolor{transfergain}{21.3}} & \makebox[2em][r]{\textcolor{transfergain}{17.7}} & \makebox[2em][r]{\textcolor{transfergain}{25.1}} & \makebox[2em][r]{14.0} & \makebox[2em][r]{\textcolor{transferloss}{12.0}} & \makebox[2em][r]{\textcolor{transferloss}{13.8}} \\
\addlinespace[3pt]
\rowcolor{transferband}\multicolumn{8}{@{}l}{\strut\textcolor{transferink}{\textbf{Coding}}} \\
\addlinespace[2pt]
LiveCodeBench & \makebox[2em][r]{29.4} & \makebox[2em][r]{\textcolor{transferloss}{18.8}} & \makebox[2em][r]{\textcolor{transfergain}{35.3}} & \makebox[2em][r]{\textcolor{transfergain}{36.7}} & \makebox[2em][r]{30.5} & \makebox[2em][r]{\textcolor{transferloss}{8.9}} & \makebox[2em][r]{\textcolor{transfergain}{30.8}} \\
CodeContests & \makebox[2em][r]{21.0} & \makebox[2em][r]{\textcolor{transferloss}{11.9}} & \makebox[2em][r]{\textcolor{transfergain}{23.6}} & \makebox[2em][r]{\textcolor{transfergain}{30.2}} & \makebox[2em][r]{24.4} & \makebox[2em][r]{\textcolor{transferloss}{6.9}} & \makebox[2em][r]{\textcolor{transferloss}{21.0}} \\
DELTA & \makebox[2em][r]{11.4} & \makebox[2em][r]{\textcolor{transferloss}{3.0}} & \makebox[2em][r]{\textcolor{transferloss}{11.2}} & \makebox[2em][r]{\textcolor{transferloss}{11.4}} & \makebox[2em][r]{9.6} & \makebox[2em][r]{\textcolor{transferloss}{0.5}} & \makebox[2em][r]{\textcolor{transfergain}{13.9}} \\
\addlinespace[3pt]
\rowcolor{transferband}\multicolumn{8}{@{}l}{\strut\textcolor{transferink}{\textbf{Knowledge and Planning}}} \\
\addlinespace[2pt]
GPQA-Diamond & \makebox[2em][r]{42.4} & \makebox[2em][r]{\textcolor{transfergain}{56.1}} & \makebox[2em][r]{\textcolor{transfergain}{47.3}} & \makebox[2em][r]{\textcolor{transfergain}{49.7}} & \makebox[2em][r]{38.1} & \makebox[2em][r]{\textcolor{transferloss}{32.8}} & \makebox[2em][r]{\textcolor{transferloss}{35.9}} \\
ACPBench & \makebox[2em][r]{54.3} & \makebox[2em][r]{\textcolor{transfergain}{61.8}} & \makebox[2em][r]{\textcolor{transfergain}{66.1}} & \makebox[2em][r]{\textcolor{transfergain}{72.1}} & \makebox[2em][r]{69.3} & \makebox[2em][r]{\textcolor{transferloss}{49.6}} & \makebox[2em][r]{\textcolor{transferloss}{65.0}} \\
HEAD-QA & \makebox[2em][r]{36.1} & \makebox[2em][r]{\textcolor{transfergain}{36.4}} & \makebox[2em][r]{\textcolor{transfergain}{37.0}} & \makebox[2em][r]{\textcolor{transfergain}{36.7}} & \makebox[2em][r]{39.2} & \makebox[2em][r]{\textcolor{transferloss}{34.3}} & \makebox[2em][r]{\textcolor{transferloss}{38.7}} \\
ToolUse & \makebox[2em][r]{30.9} & \makebox[2em][r]{\textcolor{transferloss}{4.1}} & \makebox[2em][r]{\textcolor{transferloss}{22.7}} & \makebox[2em][r]{\textcolor{transferloss}{20.6}} & \makebox[2em][r]{43.3} & \makebox[2em][r]{\textcolor{transferloss}{9.3}} & \makebox[2em][r]{\textcolor{transfergain}{52.6}} \\
Science & \makebox[2em][r]{42.2} & \makebox[2em][r]{\textcolor{transferloss}{0.0}} & \makebox[2em][r]{\textcolor{transfergain}{46.4}} & \makebox[2em][r]{\textcolor{transfergain}{47.7}} & \makebox[2em][r]{36.7} & \makebox[2em][r]{\textcolor{transferloss}{0.0}} & \makebox[2em][r]{\textcolor{transferloss}{32.5}} \\
\addlinespace[3pt]
\rowcolor{transferband}\multicolumn{8}{@{}l}{\strut\textcolor{transferink}{\textbf{General Capabilities}}} \\
\addlinespace[2pt]
CoQA & \makebox[2em][r]{11.8} & \makebox[2em][r]{\textcolor{transferloss}{1.6}} & \makebox[2em][r]{\textcolor{transfergain}{53.6}} & \makebox[2em][r]{\textcolor{transfergain}{61.1}} & \makebox[2em][r]{52.3} & \makebox[2em][r]{\textcolor{transferloss}{1.6}} & \makebox[2em][r]{\textcolor{transfergain}{78.6}} \\
IFEval & \makebox[2em][r]{66.2} & \makebox[2em][r]{\textcolor{transferloss}{22.0}} & \makebox[2em][r]{\textcolor{transferloss}{59.6}} & \makebox[2em][r]{\textcolor{transferloss}{56.7}} & \makebox[2em][r]{69.3} & \makebox[2em][r]{\textcolor{transferloss}{23.9}} & \makebox[2em][r]{\textcolor{transferloss}{68.1}} \\
HaluEval & \makebox[2em][r]{35.8} & \makebox[2em][r]{\textcolor{transferloss}{0.0}} & \makebox[2em][r]{\textcolor{transferloss}{29.1}} & \makebox[2em][r]{\textcolor{transferloss}{30.6}} & \makebox[2em][r]{61.2} & \makebox[2em][r]{\textcolor{transferloss}{0.0}} & \makebox[2em][r]{\textcolor{transferloss}{60.2}} \\
MC-TACO & \makebox[2em][r]{50.0} & \makebox[2em][r]{\textcolor{transfergain}{66.7}} & \makebox[2em][r]{\textcolor{transfergain}{63.5}} & \makebox[2em][r]{\textcolor{transfergain}{71.7}} & \makebox[2em][r]{76.5} & \makebox[2em][r]{\textcolor{transferloss}{67.4}} & \makebox[2em][r]{\textcolor{transferloss}{68.1}} \\
WildJailbreak & \makebox[2em][r]{29.4} & \makebox[2em][r]{\textcolor{transferloss}{3.9}} & \makebox[2em][r]{\textcolor{transferloss}{21.0}} & \makebox[2em][r]{\textcolor{transferloss}{24.7}} & \makebox[2em][r]{27.3} & \makebox[2em][r]{\textcolor{transferloss}{6.2}} & \makebox[2em][r]{\textcolor{transferloss}{25.6}} \\
WildGuardTest & \makebox[2em][r]{74.0} & \makebox[2em][r]{\textcolor{transferloss}{24.5}} & \makebox[2em][r]{\textcolor{transferloss}{66.3}} & \makebox[2em][r]{\textcolor{transferloss}{70.6}} & \makebox[2em][r]{77.1} & \makebox[2em][r]{\textcolor{transferloss}{48.9}} & \makebox[2em][r]{\textcolor{transferloss}{75.9}} \\
\bottomrule
\end{tabular}
\par
\begin{minipage}{\linewidth}
\fontsize{8}{9.5}\selectfont
\textcolor{transfergain}{Teal}: gain; \textcolor{transferloss}{rust}: loss. OMEGA averages three splits; DELTA reports
BouncingSim test pass rate. IFEval uses strict instruction-level accuracy.
\end{minipage}
\endgroup
\end{table}

%% file: arxiv/tables/passn_full.tex
% Generated from arxiv/figures/pass1024_scores_snapshot_2026-09-15.json
\begin{table}[H]
\centering
\caption{Detailed breakdown of Pass@$N$ results for the RL and base models in
\Cref{fig:passn-qwen-gemma}.}
\label{tab:passn-full}
\begingroup
\definecolor{passnheader}{HTML}{31445A}
\definecolor{passnfamily}{HTML}{E7EDF2}
\definecolor{passnbase}{HTML}{F2F4F6}
\definecolor{passngain}{HTML}{DDF1EA}
\definecolor{passngainink}{HTML}{176B57}
\definecolor{passnloss}{HTML}{F7E3DE}
\definecolor{passnlossink}{HTML}{9A4033}
\definecolor{passnneutral}{HTML}{ECEFF1}
\newcommand{\passngaincell}[1]{\cellcolor{passngain}\textcolor{passngainink}{#1}}
\newcommand{\passnlosscell}[1]{\cellcolor{passnloss}\textcolor{passnlossink}{#1}}
\newcommand{\passnneutralcell}[1]{\cellcolor{passnneutral}#1}
\small
\renewcommand{\arraystretch}{1.10}
\begin{tabular*}{\linewidth}{@{\extracolsep{\fill}}ll*{6}{r}@{}}
\toprule
\textbf{Benchmark} & \textbf{Model}
& \multicolumn{6}{c}{\textbf{Pass@}$N$\textbf{ (\%)}} \\
\cmidrule(l){3-8}
& & \textbf{1} & \textbf{4} & \textbf{16} & \textbf{64} & \textbf{256} & \textbf{1024} \\
\midrule
\rowcolor{passnfamily}
\multicolumn{8}{@{}l}{\strut\textcolor{passnheader}{\textbf{Qwen3-14B-Base}}} \\
\addlinespace[1pt]
\rowcolor{passnbase}\textbf{MATH500} & \textcolor{passnheader}{Base} & 73.00 & 86.24 & 91.36 & 94.17 & 95.92 & 97.80 \\
 & Math RL & \passngaincell{87.13} & \passngaincell{91.23} & \passngaincell{93.42} & \passngaincell{94.84} & \passnlosscell{95.71} & \passnlosscell{96.60} \\
 & Math-Code RL & \passngaincell{88.13} & \passngaincell{91.95} & \passngaincell{93.98} & \passngaincell{94.85} & \passnlosscell{95.35} & \passnlosscell{95.80} \\
\addlinespace[2pt]
\rowcolor{passnbase}\textbf{AIME 2024} & \textcolor{passnheader}{Base} & 15.66 & 32.31 & 51.42 & 66.11 & 75.90 & 90.00 \\
 & Math RL & \passngaincell{33.45} & \passngaincell{54.16} & \passngaincell{67.22} & \passngaincell{74.20} & \passngaincell{78.89} & \passnlosscell{83.33} \\
 & Math-Code RL & \passngaincell{44.42} & \passngaincell{62.24} & \passngaincell{71.08} & \passngaincell{76.70} & \passngaincell{81.15} & \passnlosscell{86.67} \\
\addlinespace[2pt]
\rowcolor{passnbase}\textbf{AIME 2025} & \textcolor{passnheader}{Base} & 12.01 & 22.79 & 33.23 & 45.97 & 62.32 & 76.67 \\
 & Math RL & \passngaincell{24.03} & \passngaincell{36.59} & \passngaincell{51.24} & \passngaincell{64.58} & \passngaincell{74.47} & \passngaincell{83.33} \\
 & Math-Code RL & \passngaincell{31.19} & \passngaincell{45.61} & \passngaincell{58.38} & \passngaincell{67.91} & \passngaincell{79.07} & \passngaincell{90.00} \\
\addlinespace[2pt]
\rowcolor{passnbase}\textbf{CodeContests} & \textcolor{passnheader}{Base} & 19.74 & 32.32 & 42.66 & 51.33 & 57.92 & 63.13 \\
 & Math RL & \passngaincell{25.83} & \passngaincell{36.80} & \passngaincell{46.27} & \passngaincell{53.17} & \passngaincell{58.81} & \passngaincell{63.66} \\
 & Math-Code RL & \passngaincell{31.63} & \passngaincell{41.74} & \passngaincell{48.40} & \passngaincell{52.37} & \passnlosscell{56.08} & \passnlosscell{60.48} \\
\addlinespace[2pt]
\rowcolor{passnbase}\textbf{GPQA-Diamond} & \textcolor{passnheader}{Base} & 36.70 & 69.10 & 91.05 & 97.79 & 99.84 & 100.00 \\
 & Math RL & \passngaincell{45.95} & \passngaincell{70.72} & \passnlosscell{84.49} & \passnlosscell{92.38} & \passnlosscell{95.82} & \passnlosscell{97.47} \\
 & Math-Code RL & \passngaincell{45.41} & \passngaincell{70.67} & \passnlosscell{85.07} & \passnlosscell{92.73} & \passnlosscell{96.22} & \passnlosscell{98.48} \\
\addlinespace[2pt]
\rowcolor{passnbase}\textbf{BoxNet} & \textcolor{passnheader}{Base} & 3.83 & 9.89 & 17.90 & 25.20 & 30.08 & 34.00 \\
 & Math RL & \passngaincell{9.88} & \passngaincell{20.67} & \passngaincell{28.66} & \passngaincell{33.88} & \passngaincell{40.45} & \passngaincell{51.00} \\
 & Math-Code RL & \passngaincell{11.24} & \passngaincell{21.85} & \passngaincell{30.25} & \passngaincell{36.58} & \passngaincell{47.59} & \passngaincell{67.00} \\
\midrule
\rowcolor{passnfamily}
\multicolumn{8}{@{}l}{\strut\textcolor{passnheader}{\textbf{Gemma-3-12B-IT}}} \\
\addlinespace[1pt]
\rowcolor{passnbase}\textbf{MATH500} & \textcolor{passnheader}{Base} & 81.31 & 87.60 & 90.62 & 92.45 & 93.88 & 94.80 \\
 & Math RL & \passngaincell{82.73} & \passngaincell{88.31} & \passngaincell{91.24} & \passngaincell{93.24} & \passngaincell{94.21} & \passnlosscell{94.40} \\
\addlinespace[2pt]
\rowcolor{passnbase}\textbf{AIME 2024} & \textcolor{passnheader}{Base} & 24.37 & 39.25 & 50.69 & 62.13 & 69.67 & 80.00 \\
 & Math RL & \passngaincell{27.76} & \passngaincell{43.21} & \passngaincell{54.77} & \passngaincell{63.13} & \passngaincell{70.50} & \passnlosscell{76.67} \\
\addlinespace[2pt]
\rowcolor{passnbase}\textbf{AIME 2025} & \textcolor{passnheader}{Base} & 18.13 & 27.34 & 36.17 & 45.13 & 57.42 & 66.67 \\
 & Math RL & \passngaincell{20.54} & \passngaincell{29.82} & \passngaincell{39.40} & \passngaincell{47.81} & \passnlosscell{55.10} & \passnlosscell{60.00} \\
\addlinespace[2pt]
\rowcolor{passnbase}\textbf{CodeContests} & \textcolor{passnheader}{Base} & 23.27 & 34.12 & 42.32 & 48.20 & 53.02 & 55.70 \\
 & Math RL & \passnlosscell{21.24} & \passnlosscell{31.23} & \passnlosscell{39.82} & \passnlosscell{45.61} & \passnlosscell{49.82} & \passnlosscell{54.91} \\
\addlinespace[2pt]
\rowcolor{passnbase}\textbf{GPQA-Diamond} & \textcolor{passnheader}{Base} & 39.01 & 64.47 & 82.58 & 90.94 & 94.83 & 96.97 \\
 & Math RL & \passnlosscell{38.59} & \passngaincell{66.14} & \passngaincell{85.24} & \passngaincell{93.22} & \passngaincell{96.40} & \passngaincell{97.47} \\
\addlinespace[2pt]
\rowcolor{passnbase}\textbf{BoxNet} & \textcolor{passnheader}{Base} & 9.08 & 11.41 & 12.84 & 14.14 & 15.02 & 16.00 \\
 & Math RL & \passnlosscell{7.31} & \passngaincell{14.39} & \passngaincell{19.19} & \passngaincell{22.19} & \passngaincell{25.33} & \passngaincell{30.00} \\
\bottomrule
\end{tabular*}
\par
\begin{minipage}{\linewidth}
\footnotesize
\colorbox{passngain}{\textcolor{passngainink}{\strut RL gain}}\quad
\colorbox{passnloss}{\textcolor{passnlossink}{\strut RL loss}}\quad
\end{minipage}
\endgroup
\end{table}

%% file: arxiv/figures/qwen3_14b_selected_five_benchmarks_figure.tex
\begin{figure}[H]
\centering
\includegraphics[width=\linewidth]
{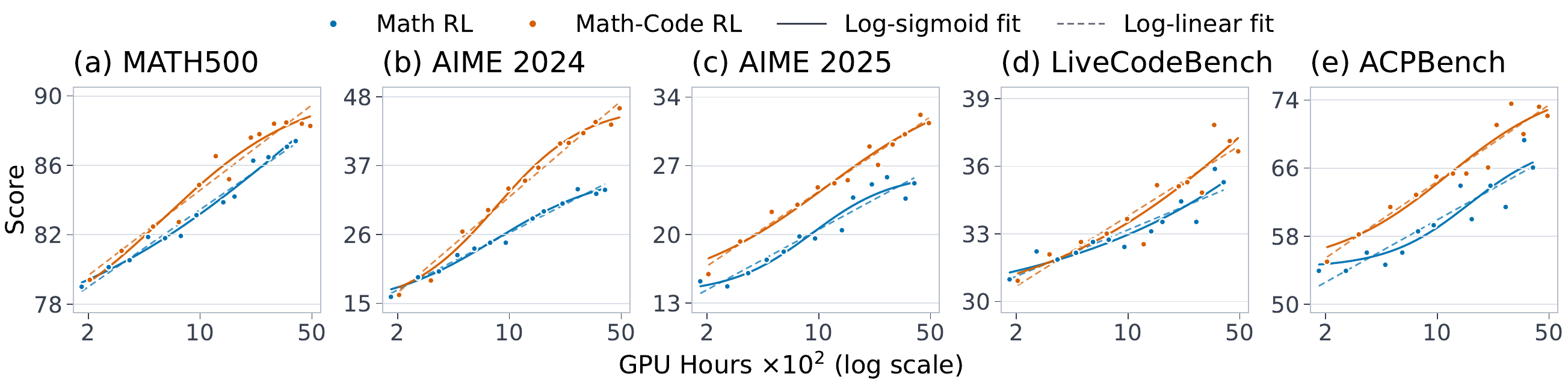}
\caption{\textbf{Scaling behavior of RL compute.} Both laws fit the observed
training regime effectively.}
\label{fig:qwen-scaling-five}
\end{figure}

%% file: arxiv/tables/qwen_scaling_five.tex
\begin{table}[H]
\centering
\caption{\textbf{Log-sigmoid and log-linear fits using GPU hours.}
$A$ is the fitted log-sigmoid asymptote on the raw $[0,1]$ score scale.
$C_{\rm mid}$ and $C_{\rm ref}$
are in units of $10^2$ GPU-hours. Both laws fit the observed training regime effectively.}
\label{tab:qwen-scaling-five}
\begingroup
\small
\renewcommand{\arraystretch}{1.10}
\arrayrulecolor{fitink}
\begin{tabular*}{\linewidth}{@{\extracolsep{\fill}}llS[table-format=1.2]S[table-format=1.3]S[table-format=2.1]>{\columncolor{fitmetric}}S[table-format=1.3]@{\hspace{8pt}}S[table-format=2.2]S[table-format=2.1]>{\columncolor{fitmetric}}S[table-format=1.3]@{}}
\toprule
& & \multicolumn{4}{c}{\textcolor{fitink}{\textbf{Log-sigmoid}}}
& \multicolumn{3}{c}{\textcolor{fitink}{\textbf{Log-linear}}} \\
\cmidrule(lr){3-6}\cmidrule(l){7-9}
\rowcolor{fitink!14}
\multicolumn{1}{c}{\textbf{RL data}} & \multicolumn{1}{c}{\textbf{Benchmark}}
& \multicolumn{1}{c}{$B$} & \multicolumn{1}{c}{$A$} & \multicolumn{1}{c}{$C_{\rm mid}$}
& \multicolumn{1}{c}{\cellcolor{fitmetric}\textcolor{fitmetricink}{\textbf{Fit $\mathsf{R^2}$}}}
& \multicolumn{1}{c}{$B_{\rm lin}(\times 100)$} & \multicolumn{1}{c}{$C_{\rm ref}$}
& \multicolumn{1}{c}{\cellcolor{fitmetric}\textcolor{fitmetricink}{\textbf{Fit $\mathsf{R^2}$}}} \\
\midrule
Math RL & MATH500 & 0.54 & 1.000 & 41.7 & 0.982 & 2.76 & 0.6 & 0.975 \\
Math RL & AIME 2024 & 1.23 & 0.361 & 8.1 & 0.981 & 5.66 & 1.2 & 0.982 \\
Math RL & AIME 2025 & 1.71 & 0.262 & 9.1 & 0.942 & 3.83 & 1.8 & 0.920 \\
Math RL & LiveCodeBench & 0.39 & 1.000 & \multicolumn{1}{c}{${>}10^4$} & 0.881 & 1.28 & 0.5 & 0.845 \\
Math RL & ACPBench & 1.69 & 0.693 & 15.8 & 0.861 & 4.56 & 2.9 & 0.851 \\
\midrule
Math-Code RL & MATH500 & 1.00 & 0.904 & 5.9 & 0.965 & 3.05 & 0.5 & 0.948 \\
Math-Code RL & AIME 2024 & 1.56 & 0.466 & 8.3 & 0.987 & 9.52 & 1.6 & 0.981 \\
Math-Code RL & AIME 2025 & 0.96 & 0.358 & 11.2 & 0.962 & 4.72 & 1.1 & 0.973 \\
Math-Code RL & LiveCodeBench & 0.59 & 0.485 & 89.3 & 0.873 & 1.94 & 1.0 & 0.857 \\
Math-Code RL & ACPBench & 1.19 & 0.762 & 11.7 & 0.926 & 5.59 & 1.6 & 0.934 \\
\bottomrule
\end{tabular*}
\arrayrulecolor{black}
\endgroup
\end{table}

%% file: arxiv/figures/qwen3_14b_selected_five_benchmarks_log_steps_figure.tex
\begin{figure}[H]
\centering
\includegraphics[width=\linewidth]
{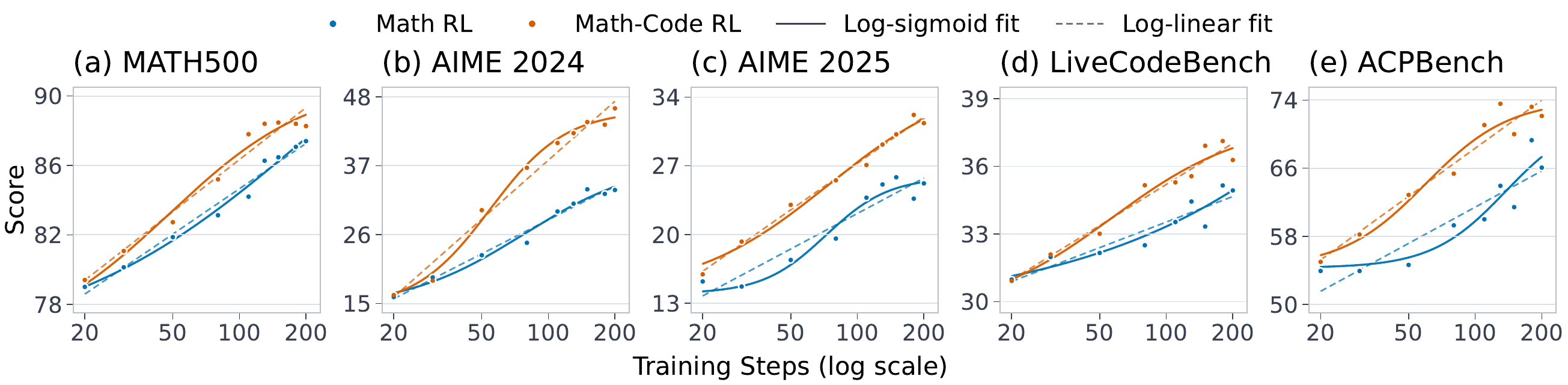}
\caption{\textbf{Training-step scaling on five benchmarks.}
Points show measured Qwen3-14B scores for Math RL and Math-Code RL.
Solid and dashed curves show log-sigmoid and log-linear fits as functions of
optimizer step, respectively. Each panel
uses its own truncated score scale.}
\label{fig:qwen-scaling-five-steps}
\end{figure}

%% file: arxiv/figures/qwen3_14b_pre200_step_extrapolation_figure.tex
\begin{figure}[H]
\centering
\includegraphics[width=\linewidth]
{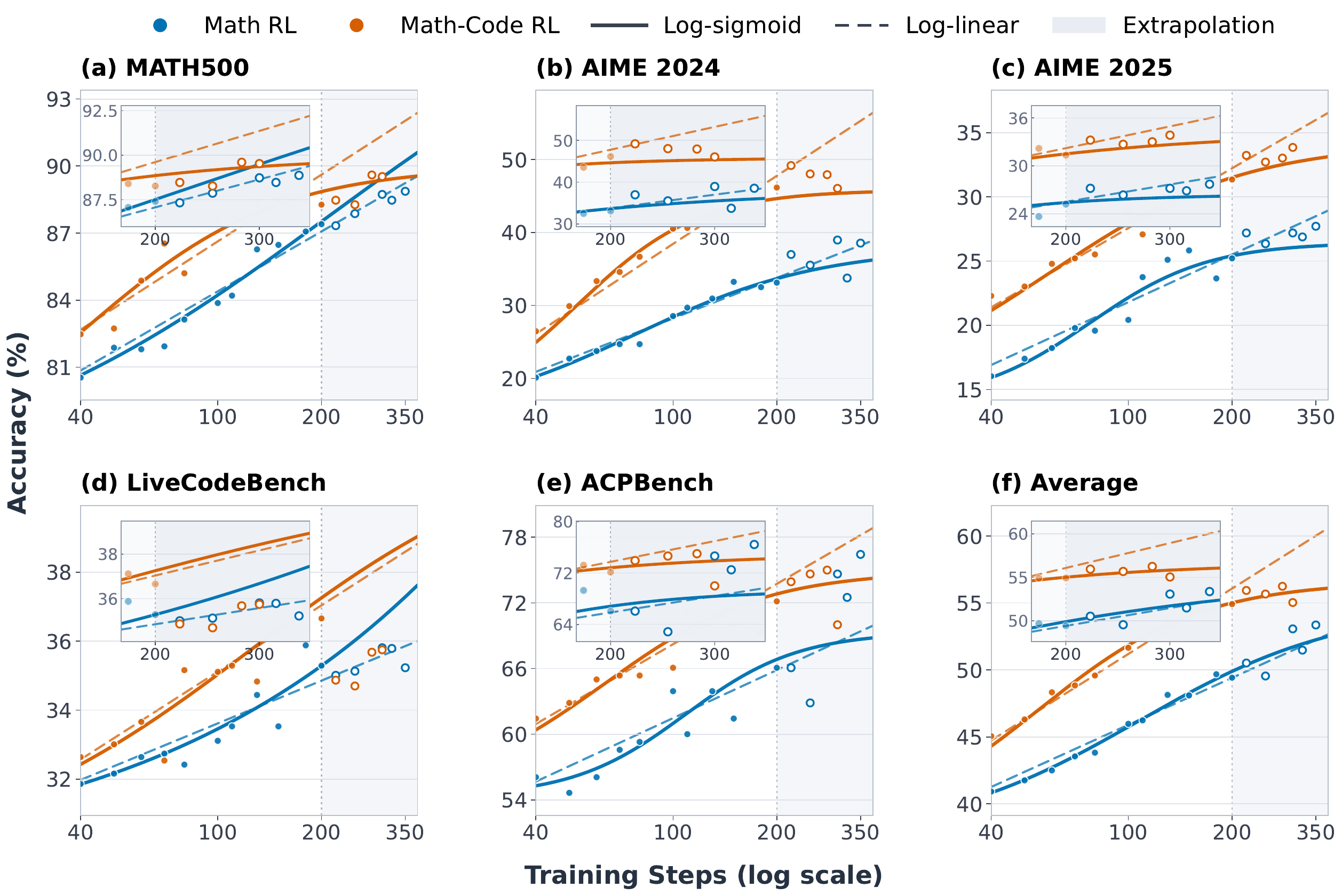}
\caption{\textbf{Scaling behavior using optimizer steps.} Both laws are fit
through step 200 and extrapolated through step 350, producing the same
qualitative pattern as fits based on GPU-hours.}
\label{fig:qwen-scaling-extrapolation-steps}
\end{figure}

%% file: arxiv/tables/qwen_scaling_five_steps.tex
\begin{table}[H]
\centering
\caption{\textbf{Log-sigmoid and log-linear fits using optimizer step.}
$A$ is the fitted log-sigmoid asymptote on the raw $[0,1]$ score scale.
$C_{\rm mid}$ and $C_{\rm ref}$
are in steps.}
\label{tab:qwen-scaling-five-steps}
\begingroup
\small
\renewcommand{\arraystretch}{1.10}
\arrayrulecolor{fitink}
\begin{tabular*}{\linewidth}{@{\extracolsep{\fill}}llS[table-format=1.2]S[table-format=1.3]S[table-format=3.1]>{\columncolor{fitmetric}}S[table-format=1.3]@{\hspace{8pt}}S[table-format=2.2]S[table-format=2.1]>{\columncolor{fitmetric}}S[table-format=1.3]@{}}
\toprule
& & \multicolumn{4}{c}{\textcolor{fitink}{\textbf{Log-sigmoid}}}
& \multicolumn{3}{c}{\textcolor{fitink}{\textbf{Log-linear}}} \\
\cmidrule(lr){3-6}\cmidrule(l){7-9}
\rowcolor{fitink!14}
\multicolumn{1}{c}{\textbf{RL data}} & \multicolumn{1}{c}{\textbf{Benchmark}}
& \multicolumn{1}{c}{$B$} & \multicolumn{1}{c}{$A$} & \multicolumn{1}{c}{$C_{\rm mid}$}
& \multicolumn{1}{c}{\cellcolor{fitmetric}\textcolor{fitmetricink}{\textbf{Fit $\mathsf{R^2}$}}}
& \multicolumn{1}{c}{$B_{\rm lin}(\times 100)$} & \multicolumn{1}{c}{$C_{\rm ref}$}
& \multicolumn{1}{c}{\cellcolor{fitmetric}\textcolor{fitmetricink}{\textbf{Fit $\mathsf{R^2}$}}} \\
\midrule
Math RL & MATH500 & 0.76 & 1.000 & 207.1 & 0.985 & 3.77 & 8.7 & 0.976 \\
Math RL & AIME 2024 & 1.61 & 0.379 & 77.3 & 0.979 & 7.83 & 16.5 & 0.976 \\
Math RL & AIME 2025 & 3.05 & 0.259 & 72.1 & 0.953 & 5.22 & 20.9 & 0.902 \\
Math RL & LiveCodeBench & 0.53 & 1.000 & \multicolumn{1}{c}{${>}10^4$} & 0.895 & 1.64 & 8.1 & 0.867 \\
Math RL & ACPBench & 2.60 & 0.723 & 137.2 & 0.899 & 6.14 & 31.3 & 0.827 \\
\midrule
Math-Code RL & MATH500 & 1.33 & 0.909 & 48.1 & 0.976 & 4.32 & 8.1 & 0.962 \\
Math-Code RL & AIME 2024 & 2.55 & 0.458 & 54.5 & 0.992 & 13.57 & 17.6 & 0.977 \\
Math-Code RL & AIME 2025 & 1.40 & 0.362 & 74.2 & 0.983 & 6.84 & 14.3 & 0.990 \\
Math-Code RL & LiveCodeBench & 1.44 & 0.380 & 56.9 & 0.961 & 2.63 & 11.1 & 0.956 \\
Math-Code RL & ACPBench & 2.29 & 0.740 & 59.7 & 0.956 & 8.15 & 17.8 & 0.944 \\
\bottomrule
\end{tabular*}
\arrayrulecolor{black}
\endgroup
\end{table}

%% file: arxiv/appendix/reasoning_route_analysis.tex
\section{Reasoning-Route Analysis}
\label{app:reasoning-mode-analysis}

This appendix provides implementation details, examples, and prompt templates
for the reasoning-route analysis in \Cref{sec:reasoning-mode-analysis}.

\subsection{Classification and measurement}

\paragraph{Questions and generation setup.}
The analysis in \Cref{sec:reasoning-mode-analysis} compares Qwen3-14B-Base
with Qwen3-Math-RL at step 200 using 16 responses per model for each of
240 questions from MATH500, AIME 2024, AIME 2025, GPQA-Diamond,
CodeContests, and LiveCodeBench. Responses use temperature 0.7, top-$p$ 0.95,
no top-$k$ truncation
($\text{top-}k=-1$), and a maximum generation length of 16,384 tokens.
The 16 responses per model and question use sampling seeds 0 through 15.

\paragraph{Route definition.}
Following \Cref{sec:reasoning-mode-analysis}, a reasoning route is the
strategy a response uses to reach its conclusion. A route is
\emph{identifiable} when the response contains enough reasoning or code about
the question to identify a solution method. The method need not be correct
or complete: a program with errors or an unfinished solution can reveal a
route, whereas a final answer alone does not. Each identifiable response
receives one label for its primary route: the method used to reach its final
conclusion or produce its submitted program. If this is unclear, we choose
the most developed method, breaking ties by choosing the first method
supported by reasoning or code. Routes differ in how they represent the
problem or in their main idea, not in notation, mistakes within the same
method, or whether it is carried out by hand or in code. For example,
mathematical simplification and a search through all possible cases are
different routes.

\paragraph{LLM-assisted classification.}
We use a two-stage process to identify each response's strategy and group
responses that follow the same route. Before classification, we remove
generated follow-up user/assistant turns and replace mathematical answer
statements and values, or GPQA choice labels, with placeholders.
In \textbf{Stage A}, GPT-5.6-Luna with high reasoning effort reads each
response separately, describes its primary strategy, and selects supporting
parts of its reasoning or code. These descriptions form the input to
\textbf{Stage B}, where the same model groups base and RL responses by
strategy within each question and assigns a shared route label to each
group. Stage B uses the descriptions rather than the original responses,
with model identity, correctness, and explicit final-answer labels hidden.
Both stages use fixed prompts for each subject area and a new Codex CLI
session (\texttt{codex exec}) for each labeling call. For quality control,
GPT-5.6-Sol with high reasoning effort independently checks whether the
Stage A descriptions are supported by the response text and whether Stage B
has incorrectly combined or separated routes.
Descriptions and labels that fail these checks are corrected and checked
again. Shortened prompt templates appear in \Cref{app:route-prompts}.

\paragraph{Route counts with equal sample sizes.}
For \Cref{fig:reasoning-routes}(a), we compare the number of distinct routes
using equal numbers of identifiable responses from both models. For each
question, we set the common sample size to the smaller number of identifiable
responses available from the two models. We average the number of distinct routes over all
subsets of that size for each model, then average equally over questions
within each benchmark and across the full set. When both models have 16
identifiable responses, this simply counts the distinct routes among all
16 responses. The overall mean weights benchmarks by their question counts.

\paragraph{Routes not found in base samples.}
For \Cref{fig:reasoning-routes}(b), we compute the fraction of each
question's 16 RL responses whose primary route does not appear in its original
16 base responses, then average over questions. The overall fraction is
487 of 3,840 RL responses (12.7\%).

\paragraph{Confidence intervals.}
We use 5,000 bootstrap resamples of questions for each
benchmark and the full set. Questions are sampled with replacement,
keeping their base and RL measurements together. The 95\% confidence
intervals use the 2.5th and 97.5th percentiles of the bootstrap estimates
for each panel's means and the paired mean difference reported in
\Cref{sec:reasoning-mode-analysis}. Route labels stay unchanged
during resampling.

\subsection{Examples of route changes after RL}
\label{app:route-examples}

These are the two examples discussed in \Cref{sec:reasoning-mode-analysis}.
Both have 16 identifiable responses per model, so the route counts use
all responses without sample-size adjustment.

\paragraph{Concentration on a route found in base samples.}
In \texttt{livecodebench-0088}, the task is to minimize
$|\mathrm{nums}[i]-\mathrm{nums}[j]|$ subject to $|i-j|\ge x$.
Eight base responses binary-search a sorted list and compare neighboring
values, six compare index pairs with nested loops, and two use indices
sorted by value. All 16 RL responses use the binary-search route, reducing
the observed route count from three to one.

\paragraph{Use of a route not found in the base sample.}
In \texttt{livecodebench-0163}, the task is to find positions where pattern
$a$ appears within distance $k$ of a position where pattern $b$ appears.
All 16 base responses compare pairs of matching positions directly.
Five RL responses use the same method, while 11 use binary search over the
list of positions where $b$ appears to find the lower end of $[i-k,i+k]$.
This gives one observed route for base and two for RL; the 11 binary-search
responses are counted as using a route not found in the base sample.

\input{arxiv/appendix/route_prompt_templates}

%% file: arxiv/appendix/route_prompt_templates.tex
\subsection{Shortened prompt templates}
\label{app:route-prompts}

These templates summarize the instructions for both stages and describe the
output fields in plain language.

\begingroup
\definecolor{routePromptA}{HTML}{245B91}
\definecolor{routePromptB}{HTML}{176B57}
\definecolor{routePromptInput}{HTML}{A05A16}
\definecolor{routePromptOutput}{HTML}{75549C}
\setlength{\fboxsep}{8pt}
\setlength{\fboxrule}{0.4pt}
\noindent\fcolorbox{routePromptA!55}{routePromptA!5}{%
\begin{minipage}{\dimexpr\linewidth-2\fboxsep-2\fboxrule\relax}
\small\raggedright
\textcolor{routePromptA}{\textbf{Stage A: describe one response's strategy (shortened)}}
\par\medskip
\textcolor{routePromptInput}{\textbf{Input.}} The question, one cleaned response, and its response ID.
\begin{enumerate}[font=\color{routePromptA}\bfseries,leftmargin=*,itemsep=2pt,parsep=0pt,topsep=5pt]
\item Read the entire response on its own. Describe only methods it carries
out or specifically proposes; do not solve the question or judge correctness.
\item Choose the method used for the final conclusion or submitted program.
If unclear, choose the most developed method, then the first method supported
by reasoning or code.
\item Describe the initial approach, problem representation, and main solution
step, without final-answer values or quality judgments.
\item Copy supporting reasoning or code exactly and link each claim to its
supporting text. Mark the response as ``no route'' if no method is attempted,
or ``unclear'' if the method cannot be identified.
\end{enumerate}
\textcolor{routePromptOutput}{\textbf{Output (main fields).}} Return a JSON list with one record containing
the response ID, route status (one route, a route change, no route, or unclear),
how the primary route was selected, starting approach, problem representation,
main solution step, short strategy summary, and supporting text linked to
each claim.
\end{minipage}}
\par\medskip
\noindent\fcolorbox{routePromptB!55}{routePromptB!5}{%
\begin{minipage}{\dimexpr\linewidth-2\fboxsep-2\fboxrule\relax}
\small\raggedright
\textcolor{routePromptB}{\textbf{Stage B: group responses into routes (shortened)}}
\par\medskip
\textcolor{routePromptInput}{\textbf{Input.}} The question and its ID, plus the Stage A strategy descriptions
for all base and RL responses to that question.
\begin{enumerate}[font=\color{routePromptB}\bfseries,leftmargin=*,itemsep=2pt,parsep=0pt,topsep=5pt]
\item Use these descriptions to group responses by their primary strategy.
Do not infer model identity or correctness.
\item Group responses under the same route when they differ only in notation,
calculation errors, completeness, wording, or whether the same method is carried
out by hand or in code. Separate routes only when the initial approach, problem
representation, or main solution step differs.
\item Assign every response ID to exactly one group. Choose group members
before naming routes, and describe the strategy shared by every member.
\item Place responses marked ``no route'' or ``unclear'' in a separate group
for responses without an identifiable route.
\end{enumerate}
\textcolor{routePromptOutput}{\textbf{Output.}} Return a JSON record containing the question ID and a list
of route groups. For each group, give its name, a one-sentence strategy
description, and the IDs of its responses. Copy every response ID exactly.
\end{minipage}}
\par
\endgroup

%% file: arxiv/appendix/empirical_case_studies.tex
\section{Empirical Analysis of Theoretical Predictions}
\label{app:empirical-cases}

This appendix examines empirical evidence for the transfer, coverage, and
scaling predictions of \Cref{sec:transfer,sec:passn,sec:scaling}, respectively.

\subsection{Transfer and forgetting across tasks and models}
\label{app:transfer-output-analysis}

We examine the task- and model-dependent transfer and forgetting predicted
by \Cref{sec:transfer}, using cases that illustrate differences in
target-task requirements and teacher--student compatibility.

\paragraph{Protocol violations in tool use.}
ToolUse requires a specified \texttt{Action} name and structured
\texttt{Action Input} arguments~\citep{shenfeld2026selfdistillation}.
One prompt requests five facts about the physical characteristics of
axolotls, together with source URLs. Qwen3-14B-Base produces one
\texttt{Action: getAxolotlFacts} call, with arguments specifying the category
``physical characteristics'' and a limit of five, and is graded correct.
Qwen3-Math-RL first produces the same call, then adds ``Please provide the
output from the getAxolotlFacts tool,'' simulates further interaction,
and repeats calls. These additional calls cause the response to be graded
incorrect. The change from issuing a single requested call to simulating
further interaction turns a correct initial action into an unsuccessful
response. This example helps explain the ToolUse decline in
\Cref{tab:transfer-forgetting}: additional generated interaction can reduce
performance on a task requiring a precise tool call, even as mathematical
performance improves.

\paragraph{Errors in factual verification.}
HaluEval requires a ``Yes'' or ``No'' judgment on whether a candidate
answer contains hallucinated information~\citep{li2023halueval}.
One item includes the question ``Which magazine was started first Arthur's
Magazine or First for Women?'' The accompanying background information
identifies \emph{Arthur's Magazine} as a periodical published in 1844--1846.
The candidate answer is ``First for Women was started first,''
whereas the dataset's reference answer is ``Arthur's Magazine.''
Qwen3-14B-Base correctly responds ``Yes'', while Qwen3-Math-RL incorrectly
responds ``No''. Both outputs satisfy the required label format; the error
is the acceptance of a false chronological comparison. This gives a concrete
example of the HaluEval decline in \Cref{tab:transfer-forgetting}: gains in
mathematical problem solving can coexist with poorer judgment of a supplied
answer's factual accuracy, reflecting the different requirements of the
target task.

\paragraph{Compatibility between teacher and student models.}
Qwen3-14B-Base and Gemma-3-12B-IT are trained with SFT on the same 40K
mathematics demonstrations obtained by rejection sampling from Qwen3-32B.
Gemma-3-SFT declines on all 14 OOD benchmarks, whereas Qwen3-SFT declines
on ten and improves on four (\Cref{tab:transfer-forgetting}).
Gemma's losses are also substantially larger on several benchmarks. The difference therefore concerns both the breadth and severity
of forgetting across tasks.

In \Cref{eq:sft-selected-modes}, the teacher policy determines the
demonstrated strategies, while the student's feature overlaps determine
their support on a target prompt. Identical demonstrations can therefore
lead to different sets $\mathcal S_E(x')$ for the two students.
\Cref{thm:sft-transfer} relates forgetting to whether these sets retain
strategies valid for the target, and \Cref{eq:sft-limit-policy} shows that
the probabilities within each set also depend on the student's initial
strategy preferences on training and target prompts.
The shared Qwen model family offers a plausible explanation for better
compatibility with the Qwen student: the teacher's demonstrated strategies
may be closer to its initial preferences. Imitating these demonstrations
can then preserve strategies useful beyond mathematics for Qwen while
reducing their probability in Gemma, explaining the selective transfer of
Qwen3-SFT and the broader, more severe forgetting of Gemma-3-SFT.

\subsection{Prompt-level Pass@$N$ analysis}
\label{app:passn-question-analysis}

We examine the finite-budget coverage predictions of \Cref{sec:passn},
focusing on how gains and losses on questions with rare successes change
benchmark coverage.

For the benchmark distribution $\mathcal B$, any policy $\pi$
satisfies the exact identity
\begin{equation}
\operatorname{Pass@}N(\pi)-\operatorname{Pass@}N(\pi_0)
=\E_{x'\sim\mathcal B}\left[
(1-R_{\pi_0}(x'))^N-(1-R_\pi(x'))^N
\right].
\label{eq:passn-tail-identity}
\end{equation}
Thus the aggregate reversals in \Cref{fig:passn-qwen-gemma} reflect
changes in the distribution of per-question success probabilities.
We examine per-question records to show how gains and losses on individual
questions produce these benchmark-level reversals.

\paragraph{Benchmark dependence: AIME 2024 versus AIME 2025.}
In the Qwen Pass@$N$ evaluation, 1,024 responses were sampled per question
and reused to estimate coverage at different budgets.
\Cref{tab:aime-boundary-items} lists the questions accounting for the
observed AIME 2024 and AIME 2025 coverage changes at $N=1024$.

\input{arxiv/tables/aime_boundary_items}

Under Math RL, some questions lose all observed successes while others gain
them, producing opposite coverage shifts in the two sets. On AIME 2024, Math RL has no
correct responses on two combinatorics questions for which the base has
three and one, respectively, lowering coverage from $27/30$ to $25/30$.
On AIME 2025, Math RL has 24 correct responses on an area question and one
on a recurrence question for which the base has none, raising coverage
from $23/30$ to $25/30$. All other questions retain the same observed
coverage status. Three of the four changes depend on only one to three
correct responses. These sampled success probabilities are of order $1/N$,
the scale governing coverage sensitivity in \Cref{thm:passn-regimes}.
The averaging in \Cref{eq:passn-tail-identity} explains how the same
model can shift coverage in opposite directions when the two question
sets contain different gains and losses. The exact discordant-pair test
gives a two-sided $p=0.5$ for each set. These coverage estimates are
sensitive to rare successes: for a question with zero successes in 1,024
independent draws, the one-sided 95\% upper bound on its success
probability is approximately $0.29\%$.

\paragraph{Training-data dependence: CodeContests.}
On CodeContests, the Qwen base model and Math RL model produce at
least one correct response for $238$ and $240$ of the $377$ questions,
respectively; Math-Code RL does so for $228$ questions. The base model
solves $22$ questions for which Math-Code RL produces no correct response,
while Math-Code RL solves $12$ questions for which the base produces no
correct response.
The following examples illustrate the two outcomes:
\begin{itemize}
\item \emph{Loss of coverage.}
\texttt{codecontests-1192} requires summing pairwise absolute differences
for up to $10^5$ inputs. The base answers it correctly in $375$ of
$1{,}024$ samples. One correct program generated by the base model sorts
the inputs and computes the sum in $O(n\log n)$ time.
None of the $1{,}024$ Math-Code RL responses is correct; every generated
program times out.
\item \emph{Gain of coverage.}
\texttt{codecontests-1332} requires finding the minimum number of subset
toggles needed to turn on the lamps in each prefix of an array.
Math-Code RL answers it correctly in $74$ of $1{,}024$ samples,
whereas none of the base responses is correct.
\end{itemize}
\Cref{thm:pg-transfer} provides an interpretation of these outcomes in
terms of strategy selection. For \texttt{codecontests-1192}, a correct
base response uses a sorting strategy $z_{\rm sort}\in\mathcal M(x')$.
By \Cref{eq:pg-transfer-identity}, training changes its odds relative to
competing strategies, with each training prompt's contribution weighted by
its feature overlap with the target. Training on code problems
can therefore reduce the probability of selecting this efficient algorithm
if these weighted updates favor competing strategies.
For \texttt{codecontests-1332}, the same mechanism can increase the
probability of selecting a strategy that solves the lamp problem.
\Cref{eq:pg-target-correct-mass} expresses both effects in the vanishing-KL
limit by conditioning the base policy on $\mathcal S(x')$: this can
increase or decrease the probability assigned to valid strategies in
$\mathcal M(x')$.

Pass@1 nevertheless increases from $19.7\%$ to $31.6\%$
(\Cref{tab:passn-full}). With only $12$ newly solved questions among $377$,
most of this increase must come from higher average accuracy on the
questions solved by both models.

\paragraph{Base-model dependence: AIME 2025 and BoxNet.}
Under Math RL, sampled AIME 2025 coverage rises from $23/30$ to
$25/30$ for Qwen but falls from $20/30$ to $18/30$ for Gemma.
In the vanishing-KL limit of the transfer model, RL conditions the base
policy on strategies favored by training (\Cref{eq:pg-target-correct-mass}).
The probability of selecting
a valid strategy rises or falls according to whether valid strategies have
a larger or smaller share of probability within this set than in the
original policy. Different initial strategy preferences and feature overlaps
can therefore produce opposite changes under the same training mixture.

BoxNet also differs across base models (\Cref{tab:passn-full}): Qwen
improves Pass@1 from $3.8\%$ to $9.9\%$ and observed Pass@1024 coverage
from $34/100$ to $51/100$, while Gemma lowers Pass@1 from $9.1\%$
to $7.3\%$ but expands coverage from $16/100$ to $30/100$.
For Gemma, the 16 questions solved by both policies lose $2{,}819$
correct responses, exceeding the $1{,}011$ gained on 14 newly solved
questions. Thus successes become spread across more questions but less
frequent overall, explaining the opposing changes in the two metrics
(\Cref{eq:passn-tail-identity}).

In \texttt{boxnet-0098}, the prompt requires moving red, blue, and green
boxes to their targets, with the blue box first moved to an adjacent
square~\citep{stojanovski2026reasoning}. The base has $0/1024$ successful
responses, whereas Gemma RL has $749/1024$. In one matched pair, the base
plan delivers only the red box, while the RL plan completes all three
deliveries, including the blue box's intermediate move. Both outputs are
valid fenced JSON. This example makes the target's requirements concrete: a
useful planning strategy must account for the prerequisite move and
all three goals (\Cref{thm:pg-transfer}).

Together, these records show different question-level outcomes behind the
aggregate curves. The same Qwen Math RL model loses observed coverage
on two AIME 2024 questions but gains it on two AIME 2025 questions.
Math-Code RL increases accuracy on some CodeContests questions and
decreases it on others. Gemma's Math RL model broadens coverage on BoxNet
while lowering accuracy on questions solved by both models.
These records illustrate changes in coverage across tasks and models,
with differences in runtime outcomes, plan completeness, and the
distribution of correct responses across questions.

\subsection{Scaling fit results and interpretation}
\label{app:scaling-fit-interpretation}

We examine the empirical implications of the scaling analysis in
\Cref{sec:scaling}, focusing on the local validity of fitted laws and the
interpretation of their coefficients.

The log-sigmoid and log-linear curves in
\Cref{fig:qwen-scaling-extrapolation-gpu-hours} both agree well with the observations
over the range of compute used for fitting. Their similarity there does not
determine how scores will change beyond that range: the log-linear law maintains a constant slope
against $\log C$, whereas the log-sigmoid law bends toward $A$. The later
evaluations distinguish the two curve shapes; which law extrapolates better
varies across training runs and benchmarks. The five-benchmark average is
smoother than some individual trajectories, so its fit alone does not
establish benchmark-level predictability.

The coefficients in \Cref{tab:qwen-scaling-five} describe local score
trends. The reported $100B_{\rm lin}$ is the fitted change in score
percentage points per unit increase in $\log C$. Math-Code RL has a larger
log-linear slope on all five benchmarks; for AIME 2024 it is $9.52$ rather
than $5.66$ points. This difference characterizes the fitted range, not a
sustained rate of improvement. For the log-sigmoid law, writing
$q(C)=[1+(C_{\rm mid}/C)^B]^{-1}$ gives
$\mathrm{d}R/\mathrm{d}\log C=(A-R_0)Bq(C)(1-q(C))$; the exponent $B$
alone does not determine the local slope. The fitted $C_{\rm mid}$ is the
compute level halfway between $R_0$ and $A$, not a measured onset of
saturation.

The fitted asymptote also requires care. On MATH500, Math RL reaches the
imposed bound $A=1$ and has $C_{\rm mid}=41.7$ in units of $10^2$
GPU-hours, near the end of the fitting range; Math-Code RL has $A=0.904$
and $C_{\rm mid}=5.9$. These values describe different fitted curve shapes
but do not establish long-run performance ceilings. The
fits in \Cref{tab:qwen-scaling-five-steps} use optimizer steps to measure compute,
so their coefficients should be interpreted on that scale.

$B$ and $B_{\rm lin}$ describe the shapes of the fitted score curves.
The gain-weighted feature overlap in the theoretical log-linear result
enters the approximation bound \Cref{eq:benchmark-linear-scaling-formal}.
Longer-range forecasts require separate validation.

%% file: arxiv/tables/aime_boundary_items.tex
\begin{table}[t]
\centering
\small
\caption{AIME questions accounting for the changes in Qwen's observed
Pass@1024 coverage after Math RL. Each row gives a question's evaluation
index and a brief problem summary. The Base and Math RL columns give the
number of correct responses among 1,024 samples per question; zero
indicates no correct response.}
\label{tab:aime-boundary-items}
\begin{tabular}{@{}llp{0.56\linewidth}rr@{}}
\toprule
Benchmark & Index & Question summary & Base & Math RL \\
\midrule
AIME24 & 8 & Count maximal placements of two colors of chips on a $5\times5$ grid. & 3 & 0 \\
AIME24 & 25 & Count octagon colorings satisfying a rotation condition. & 1 & 0 \\
AIME25 & 11 & Find the area of a bounded region defined by nonlinear inequalities in the plane. & 0 & 24 \\
AIME25 & 27 & Compute a modular remainder from a long rational recurrence. & 0 & 1 \\
\bottomrule
\end{tabular}
\end{table}